\documentclass{article} % For LaTeX2e
\usepackage{iclr2026_conference,times}
\usepackage{algorithm}
\usepackage{algpseudocode}
\usepackage{algpseudocode}
\usepackage{amsmath, amssymb}
\usepackage{amsthm,refcount}
\usepackage{csquotes}

\usepackage{enumitem}
\usepackage{xspace}
\usepackage[utf8]{inputenc}
\usepackage{graphicx}
\usepackage{rotating}
\usepackage{tabularx}
\usepackage{wrapfig}
\usepackage{booktabs}
\usepackage{makecell}
\usepackage{multirow}
\usepackage{ragged2e}
\usepackage{colortbl}
\usepackage{array}
\newcolumntype{L}{>{\RaggedRight}p{0.8\textwidth}}
\renewcommand{\arraystretch}{1.1}
\small

\usepackage{xcolor}
\definecolor{lightblue}{rgb}{0.2,0.4,0.8} 

\newcommand{\rwOne}[1]{\textcolor{black}{#1}}
\newcommand{\rwTwo}[1]{\textcolor{black}{#1}}
\newcommand{\rwThree}[1]{\textcolor{black}{#1}}
\newcommand{\rwFour}[1]{\textcolor{black}{#1}}
\newcolumntype{C}[1]{>{\centering\arraybackslash}p{#1}}

\usepackage{amsmath,amsfonts,bm}

\def\eqref#1{equation~\ref{#1}}
\def\1{\bm{1}}

\DeclareMathAlphabet{\mathsfit}{\encodingdefault}{\sfdefault}{m}{sl}
\SetMathAlphabet{\mathsfit}{bold}{\encodingdefault}{\sfdefault}{bx}{n}

\usepackage{hyperref}
\usepackage{url}
\usepackage{tikz,xcolor}
\usepackage[dvipsnames]{xcolor}
\usepackage{pgfplots}
\pgfplotsset{compat=1.17}
\usepackage{subcaption}

\newtheorem{hypothesis}{Hypothesis} 

\newcounter{HypothesisCounter}

\newtheorem{theorem}{Theorem}[section]

\newtheorem{proposition}[theorem]{Proposition}
\newtheorem{corollary}[theorem]{Corollary}

\usepackage{lineno}

\newcommand{\modelname}{\textsf{ATEX-CF}\xspace}
\title{\modelname: Attack-Informed Counterfactual Explanations for Graph Neural Networks%Effective Counterfactual via Attack-Explanation Fusion for Node Classification
\vspace{0.5em}
}

\author{
\begin{tabular}{cc} % two columns
\begin{minipage}[t]{0.45\textwidth}
\centering
Yu Zhang \\
{\normalfont\small
Aalborg University \\
Aalborg, Denmark, 9220 \\
\texttt{yuzhang@cs.aau.dk}
}
\end{minipage} &
\begin{minipage}[t]{0.45\textwidth}
\centering
Sean Bin Yang \\
{\normalfont\small
Aalborg University \\
Aalborg, Denmark, 9220 \\
\texttt{seany@cs.aau.dk}
}
\end{minipage} \\[4em]
\begin{minipage}[t]{0.45\textwidth}
\centering
Arijit Khan \\
{\normalfont\small
Bowling Green State University, USA,  43403\\
Aalborg University, Denmark, 9220\\
\texttt{arijitk@bgsu.edu}
}
\end{minipage} &
\begin{minipage}[t]{0.45\textwidth}
\centering
Cuneyt Gurcan Akcora \\
{\normalfont\small
University of Central Florida \\
Orlando, FL, USA, 32816 \\
\texttt{cuneyt.akcora@ucf.edu}
}
\end{minipage}
\end{tabular}
}

\iclrfinalcopy % Uncomment for camera-ready version, but NOT for submission.

\begin{document}

\maketitle

\ificlrfinal
  \fancyhead[L]{Published as a conference paper at ICLR 2026}
\else
  \fancyhead[L]{Under review as a conference paper at ICLR 2026}
\fi
\fancyhead[R]{}
\renewcommand{\headrulewidth}{0.4pt}

\begin{abstract}
Counterfactual explanations offer an intuitive way to interpret graph neural networks (GNNs) by identifying minimal changes that alter a model’s prediction, thereby answering “{\em what must differ for a different outcome?}”. 
In this work, we propose a novel framework, \modelname that unifies adversarial attack techniques with counterfactual explanation generation—a connection made feasible by their shared goal of flipping a node’s prediction, yet differing in perturbation strategy: adversarial attacks often rely on edge additions, while counterfactual methods typically use deletions.
Unlike traditional approaches that treat explanation and attack separately, our method efficiently integrates both edge additions and deletions, grounded in theory, leveraging adversarial insights to explore impactful counterfactuals.
In addition, by jointly optimizing fidelity, sparsity, and plausibility under a constrained perturbation budget, our method produces instance-level explanations that are both informative and realistic.
Experiments on synthetic and real-world node classification benchmarks demonstrate that \modelname generates faithful, concise, and plausible explanations, highlighting the effectiveness of integrating adversarial insights into counterfactual reasoning for GNNs. \textbf{Our code is available at  \url{https://github.com/zhangyuo/ATEX_CF}.}
\end{abstract}

\section{Introduction}

Graph neural networks excel at node classification by recursively aggregating neighbor features and graph topology, yet their opaque inference undermines trust in critical applications such as healthcare, finance, and scientific discovery \citep{DBLP:journals/pacmmod/ChenQWKKG24,DBLP:journals/csur/ZhongBM25}. This limitation has spurred research into GNN explainability, with {\em counterfactual methods} \citep{yuan2022explainability,DBLP:journals/corr/abs-2505-11396,DBLP:journals/csur/PradoRomeroPSG24} in particular aiming to determine the smallest modifications to node features or graph structure that cause a model's prediction to change.

Meanwhile, \emph{adversarial attacks} \citep{DBLP:journals/tkde/ZhangZLCWKYD24,DBLP:journals/tkde/ZhuCYH24,DBLP:journals/tkde/SunDYZWYHL23} on GNNs have become an equally important line of research, as GNNs can be undermined by minimal, strategically crafted graph-structure perturbations, highlighting the need for robustness analysis. Consequently,  robustness against adversarial attacks has become a key priority in GNN research.

Traditional counterfactual graph generation methods, e.g., CF$^2$ \citep{tan2022learning}, GCFExplainer \citep{huang2023global}, primarily rely on \emph{edge deletion} to identify crucial substructures that support a particular prediction. While effective, this deletion-centric perspective overlooks the role of \emph{missing relations} in the original graph whose addition could substantially influence predictions. In parallel, extensive studies in graph adversarial learning have demonstrated that adding a small (e.g., 2) number of carefully selected edges can effectively flip the prediction of a target node \citep{DBLP:journals/tkde/ChenWSPDWCG25,DBLP:journals/tkde/ZhuCYH24}.  
Such added edges---though absent in the input graph---often correspond to semantically plausible and structurally coherent relations. 

Despite their importance, current approaches address these two directions largely in isolation. From a counterfactual reasoning perspective, adversarially added edges naturally serve as \textit{actionable candidates} for counterfactual generation: {\em They represent the minimal structural additions required to alter the model’s decision.} However, existing counterfactual methods, which predominantly rely on edge deletion, have largely overlooked the potential of incorporating edge-addition information derived from adversarial attack strategies.

Motivated by these insights, {\em we design a unified framework, \modelname \footnote{abbreviation for \textbf{\underline{At}}tack \textbf{\underline{Ex}}planation \textbf{\underline{C}}ounter\textbf{\underline{f}}actual} that incorporates attack semantics into counterfactual generation in a controlled and interpretable manner.} 
Extending counterfactual generation to include \emph{edge addition} has significant benefits. From a \textbf{quantitative} perspective, we demonstrate that edge-addition counterfactuals can (1) increase the likelihood of flipping predictions and (2) achieve this with a smaller perturbation budget. From a \textbf{qualitative} perspective, they provide practical advantages: (1) \textbf{Complementary explanatory coverage} --- while edge-deletion counterfactual identifies which existing relations are crucial for a prediction, edge-addition candidates reveal which missing relations could have altered the outcome. For example, in healthcare, a GNN may classify a patient as low-risk for heart disease due to the lack of an edge representing \textquote{symptom--drug correlation}, while introducing an edge \textquote{patient medication record $\rightarrow$ cardiac side effects} can flip the prediction and reveal hidden reasoning paths. (2) \textbf{Uncovering model bias and data deficiencies}  --- adding certain edges can divulge over-reliance on specific nodes or structural biases. For example, a paper may be misclassified as \textquote{theoretical mathematics} due to missing citation edges to authoritative AI conferences. Introducing an edge \textquote{paper $\rightarrow$ ICLR Best Paper Award} corrects the prediction, highlighting dataset limitations and model vulnerabilities.  
\begin{wrapfigure}[11]{r}{0.5\textwidth}
\vspace{-15px}
  \centering
  \includegraphics[width=0.38\textwidth]{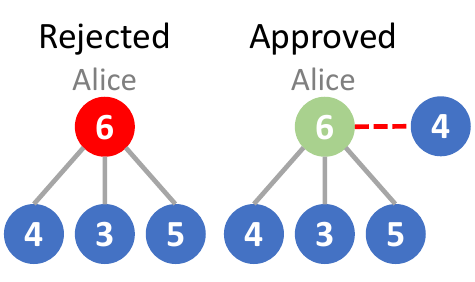}
  \caption{Illustration of counterfactual limitations in the Loan Decision dataset.}
  \label{fig:loan-case}
\end{wrapfigure}

\textbf{Case Study.}  
To illustrate the limitations of existing counterfactual methods, consider a scenario from the \textit{Loan-Decision} dataset \citep{ma2025c2explainer}.
Loan approval is granted when both conditions are met: \texttt{income} $> 5$ and \texttt{degree} $> 3$.  
Applicant \texttt{Alice} has income $6$ (satisfies condition) but degree $3$ (fails). The model predicts rejection. Classical \textbf{deletion-based} counterfactual methods fail here--removing edges further reduces degree.  
Unconstrained \textbf{edge additions} (e.g., linking to a billionaire) succeed but can be implausible.  Our method \modelname identifies a feasible peer connection that serves as an actionable update and flips the prediction.

\iffalse
\begin{figure}[H]
\centering
\begin{tikzpicture}[node distance=1.5cm]
    \node[draw, circle, fill=red!20, label=above:{\small Alice (income=6)}] (A) at (0,0) {};
    \node[draw, circle, fill=blue!20] (B) at (-1,-1) {};
    \node[draw, circle, fill=blue!20] (C) at (0,-1) {};
    \node[draw, circle, fill=blue!20] (D) at (1,-1) {};
    \path (A) edge (B); \path (A) edge (C); \path (A) edge (D);
    \node[red] at (0,0.8) {\textbf{Rejected}};
    
    \draw[->, thick] (2,0) -- (4,0);
    
    \node[draw, circle, fill=green!20, label=above:{\small Alice (income=6)}] (A2) at (5,0) {};
    \node[draw, circle, fill=blue!20] (B2) at (4,-1) {};
    \node[draw, circle, fill=blue!20] (C2) at (5,-1) {};
    \node[draw, circle, fill=blue!20] (D2) at (6,-1) {};
    \node[draw, circle, fill=blue!20, label=above:{\small
    Peer income=4}] (E) at (7.5,0) {};
    \path (A2) edge (B2); \path (A2) edge (C2); \path (A2) edge (D2);
    \path (A2) edge[green, thick] (E);
    \node[green] at (5,0.8) {\textbf{Approved}};
\end{tikzpicture}
\caption{Adding a realistic edge flips Alice's outcome from rejection to approval.}
\label{fig:loan-case}
\end{figure}
\fi

While this fusion is promising, combining adversarial attacks with counterfactual explanations is non-trivial~\cite{freiesleben2022intriguing}. Adversarial edges are optimized for misclassifications rather than interpretability, raising challenges in ensuring the qualities of a good counterfactual explanation, such as \emph{high impact}, \emph{sparsity}, and \emph{plausibility} \citep{DBLP:journals/csur/LongaASCLLP25}. 
Furthermore, when considering missing edge additions to the input graph, the search space of possible perturbations remains combinatorially large, requiring principled mechanisms to balance effectiveness with efficiency. 

\textbf{Our contributions} can be summarized as follows:
\begin{itemize}[noitemsep,nolistsep,leftmargin=*]
    \item \textbf{Unified perspective.} We establish, for the first time, a theoretical bridge between adversarial attacks and counterfactual explanations in GNNs, showing that adversarial edge additions can be repurposed as counterfactual candidates. This connection provides a principled foundation for unifying attack and explanation.
    
    \item \textbf{Hybrid counterfactual framework.} We design a novel solution, \modelname, that simultaneously leverages \emph{edge deletions} (traditional counterfactual explanations) and \emph{attack-informed edge additions} (from adversarial strategies), thereby offering a more comprehensive and actionable counterfactual than deletion-only approaches.
    
    \item \textbf{Enhanced explanatory coverage.} By incorporating edge-addition counterfactuals, \modelname uncovers missing but semantically plausible relations, complements deletion-based explanations, and enables proactive optimization (e.g., suggesting constructive graph modifications rather than only indicating critical existing edges).
    
    \item \textbf{Efficiency and controllability.} We exploit adversarial attack logistics to form a focused candidate space, significantly reducing the combinatorial complexity of our counterfactual search. In addition, \modelname integrates sparsity and plausibility constraints to ensure interpretable and realistic explanations. 
    
    \item \textbf{Empirical validation.} Through experiments on benchmark datasets, we demonstrate that \modelname improves explanatory power, maintains semantic plausibility, and reduces computational burden compared with state-of-the-art counterfactual generation and adversarial attack methods.
\end{itemize}

\section{Preliminaries}

\subsection{Node Classification and Graph Neural Networks}
\noindent\textbf{Node Classification in a Graph.}
We consider the task of node classification in a graph, denoted as $G = (V, E, \mathbf{X})$, where $V$ is a set of nodes, $E \subseteq \{(v, w) \mid v, w \in V\}$ is a set of undirected, unweighted edges, and $\mathbf{X} = \{\mathbf{x}_0, \mathbf{x}_1, \ldots, \mathbf{x}_{N-1}\}$ comprises node feature vectors with $\mathbf{x}_i \in \mathbb{R}^d$ for each node $v_i$. The adjacency matrix $\mathbf{A} \in \{0, 1\}^{N \times N}$ has entries $\mathbf{A}_{vw} = 1$ if $(v,w)\in E$ and $0$ otherwise. A subset $V_L \subseteq V$ is labeled, forming training data; each labeled node has a class $y_v \in \mathcal{C} = \{1, \ldots, c\}$. The goal is to predict the label of a target node $v \in V$ in a supervised manner given $\mathbf{A}$ and $\mathbf{X}$. Key mathematical symbols are summarized in Table~\ref{tab:notations} in the Appendix.  

\noindent\textbf{Graph Neural Networks.}
Graph Neural Networks classify nodes through a message-passing scheme \citep{kipf2017semi}. Each node representation is iteratively updated by aggregating and transforming information from its neighbors. For the Graph Convolutional Network (GCN), a prominent GNN, the hidden representation at layer $l+1$ is $\mathbf{H}^{(l+1)} = \sigma\!\left( \hat{\mathbf{A}} \mathbf{H}^{(l)} \mathbf{W}^{(l)} \right),$ 
where $\mathbf{H}^{(0)} = \mathbf{X}$, $\sigma$ is a nonlinear activation, $\mathbf{A}_{\text{self}} = \mathbf{A} + \mathbf{I}_N$ augments the adjacency with self-loops, $\mathbf{D}_{ii} = \sum_j (\mathbf{A}_{\text{self}})_{ij}$ is the degree matrix, and $\hat{\mathbf{A}} = \mathbf{D}^{-\frac{1}{2}} \mathbf{A}_{\text{self}} \mathbf{D}^{-\frac{1}{2}}$. The trainable weights at layer $l$ are $\mathbf{W}^{(l)}$. The final output is obtained by applying a softmax to the last hidden layer $\mathbf{Z} = \text{softmax}\!\left( \hat{\mathbf{A}} \mathbf{H}^{(K)} \mathbf{W}^{(K)} \right),$ 
with $\mathbf{Z} \in \mathbb{R}^{N \times c}$ giving class probability distributions. Row $\mathbf{Z}_v$ is the distribution for node $v$, and the predicted class is $\hat{y}_v = \arg\max \mathbf{Z}_v$.

\subsection{GNN Explanations}
\label{subsec:gnn-explain}
GNN explanation methods \citep{yuan2022explainability,DBLP:journals/csur/LongaASCLLP25}
reveal the structural and feature-based evidence that plays a key role in predictions. We categorize them into two paradigms:

\textbf{Factual explanations} identify subgraphs or features \emph{supporting} the original prediction. For a target node $v$, an explanation subgraph $G_v \subseteq G$ satisfies $f(G_v, \mathbf{X_v}) = f(G, \mathbf{X_v}),$ where $f$ is the GNN model and $\mathbf{X_v}$ denotes the features of $v$. The {\sf GNNExplainer}  method \citep{ying2019gnnexplainer} optimizes $G_v$ to maximize mutual information with the prediction.

\textbf{Counterfactual explanations}
identify \emph{minimal perturbations} $\Delta \mathbf{A}$ to alter target node $v$'s prediction $f(\mathbf{A}, \mathbf{X}, v) \neq f(\mathbf{A} \odot \Delta \mathbf{A}, \mathbf{X}, v), \quad 
\text{s.t. } \quad \|\Delta \mathbf{A}\|_0 \leq \kappa,$
where $\kappa$ is a perturbation budget. 

\subsection{Adversarial Attacks on GNNs}
\label{subsec:gnn-attack}

Adversarial attacks deliberately perturb graphs (including edge-based and feature-based perturbation) to mislead predictions. Key categories include i) evasion and ii) poisoning attacks. 

\textbf{Evasion attacks} modify the graph \emph{during inference} without retraining. For target node $v$, edge-based attackers solve $\max_{\Delta \mathbf{A}} \mathcal{L}(f(\mathbf{A} \odot \Delta \mathbf{A}, \mathbf{X}, v), y_v) 
\quad \text{s.t. } \quad \|\Delta \mathbf{A}\|_0 \leq \kappa,$ 
where $\mathcal{L}$ is the loss function which quantifies prediction error.

\textbf{Poisoning attacks} corrupt the \emph{training graph} to degrade retrained models. For target node $v$, edge-based attackers optimize $\max_{\Delta \mathbf{A}} \mathcal{L}(f_{\theta^*}(\mathbf{A}, \mathbf{X}, v),y_v) \quad \text{s.t.} \quad \theta^* = \arg\min_\theta \mathcal{L}(f_\theta(\mathbf{A} \odot \Delta \mathbf{A}, \mathbf{X})),\quad 
\|\Delta \mathbf{A}\|_0 \leq \kappa$.

Table~\ref{tab:compare} summarizes GNN explanations and adversarial attacks according to edge-based perturbation methods by their core characteristics.

\iffalse
\vspace{-0.5em}
\begin{table}[h]
\centering
\small
\caption{Summarizing GNN explanation and attack paradigms. 
We use $E^-$ to denote edge deletions (removing existing edges from $E$) and $E^+$ to denote edge additions (introducing new edges not present in $E$).}

\begin{tabular}{l|p{2.5cm}|p{6.8cm}}
\hline
\textbf{Category} & \textbf{Goal} & \textbf{Primary Operation \& Example} \\
\hline
Factual Explanation & Explain prediction 
& Identify key subgraph (e.g., {\sf GNNExplainer} \citep{ying2019gnnexplainer}) \\

Counterfactual Explanation & Alter prediction 
& Mainly edge deletion (E$-$) (e.g., {\sf CF-GNNExplainer} \citep{lucic2022cf}) \\

Evasion Attack & Cause misclassification 
& Edge addition/deletion and node modification, mainly edge addition (E$+$) (e.g., {\sf TDGIA} \citep{zou2021tdgia}) \\

Poisoning Attack & Degrade model 
& Edge addition/deletion and node modification, mainly edge addition (E$+$) (e.g., {\sf Nettack} \citep{zugner2018adversarial}) \\
\hline
\end{tabular}
\label{tab:compare}
\end{table}
\vspace{-1em}
\fi

\begin{table}[h]
\centering
\small
\caption{Comparison of GNN explanation and attack paradigms. 
We use $E^-$ to denote edge deletions (removing existing edges) and $E^+$ to denote edge additions (introducing new edges).}
\resizebox{\linewidth}{!}{
\begin{tabular}{l|p{2.7cm}|p{4.8cm}|p{3.5cm}}
\hline
\textbf{Category} & \textbf{Goal} & \textbf{Primary Operation} & \textbf{Example} \\
\hline
Factual Expl. 
& Explain prediction 
& Identify key subgraph 
& \tiny{{\sf GNNExplainer} \citep{ying2019gnnexplainer}} \\

Counterfactual Expl. 
& Alter prediction 
& Mainly $E^-$ (edge deletions), though some recent work includes $E^+$ 
& \tiny{{\sf CF-GNNExplainer} \citep{lucic2022cf}} \\

Evasion Attack 
& Misclassify node 
& $E^+$/$E^-$ in inference, often $E^+$ dominant 
& \tiny{{\sf TDGIA} \citep{zou2021tdgia}} \\

Poisoning Attack 
& Degrade model 
& $E^+$/$E^-$ in training, often $E^+$ dominant 
& \tiny{{\sf Nettack} \citep{zugner2018adversarial}} \\
\hline
\end{tabular}}
\label{tab:compare}
\end{table}

\noindent\textbf{Key Insight.} While counterfactual explanations have historically emphasized \textbf{$E^-$} to reveal model fragility, adversarial attacks often exploit \textbf{$E^+$} by introducing new connections. 
More importantly, the attack literature has developed efficient methods to select which edges to add/delete under small perturbation budgets (e.g., $\kappa=1,\ldots,5$), despite the combinatorially large number of possible additions in graphs, making naive counterfactual search impractical. 
This potential synergy between counterfactual reasoning and attack strategies motivates our problem formulation (\S\ref{sec:prob}) and the unified framework we propose in (\S\ref{sec:methodology}).

\rwThree{In this work, we focus on structural evasion attacks, which directly modify the graph topology at inference time. These attacks are particularly suited to our counterfactual framework, as they avoid retraining and yield interpretable perturbations that align with our goals of sparsity and plausibility.}

\subsection{Problem Formulation}
\label{sec:prob}
Given a graph $G=(V,E,\mathbf{X})$ with adjacency matrix $\mathbf{A}$ and node features $\mathbf{X}$, and a pre-trained GNN classifier $f$, our goal for a target node $v \in V$ is to find a small set of edge perturbations $\Delta \mathbf{E} = \Delta \mathbf{E}^+ \cup \Delta \mathbf{E}^-$, corresponding to additions ($\Delta \mathbf{E}^+$) and deletions ($\Delta \mathbf{E}^-$), such that the prediction for $v$ flips while the resulting counterfactual graph remains \emph{interpretable} and \emph{plausible}. This problem combines two perspectives: from the attack literature, where efficient methods have been developed to select high-impact edge additions under small budgets, and from counterfactual explanations, where minimal and semantically meaningful deletions expose decision-supporting edges. We formalize this hybrid objective in \S\ref{sec:methodology}.

\section{A Dual Approach of Explanations and Attacks for GNNs} \label{sec:synthesis}

We develop a theoretical framework that links targeted structural evasion attacks on graph neural networks with instance-level counterfactual explanation subgraphs. The core objective is to formalize when and why adversarial perturbations can serve as building blocks for counterfactual explanations. To this end, we introduce a hypothesis to capture the relationship between the attack subgraph and the counterfactual explanation of a target node. \textbf{More importantly, we \rwThree{support this hypothesis with gradient-based reasoning and empirical similarity measures} in Appendix~\ref{sec:hypotheses}}. \rwThree{To the best of our knowledge, this hypothesis and supporting evidence are presented for the first time as a plausible and empirically grounded link between adversarial perturbations and counterfactual explanations in graph learning.}

To compare explanation and attack subgraphs, we consider two forms of graph similarity: i) structural similarity~\citep{doan2021interpretable}: overlap in nodes or edges, measurable via graph edit distance, and maximum common subgraph metrics. ii) semantic similarity~\citep{bai2020learning}: closeness in learned graph-level embeddings, indicating similar functional or predictive roles even if the structures differ.  %We use $\text{Sim}(\cdot,\cdot)$ to denote a graph similarity measure by graph edit distance, maximum common subgraph, and graph embedding vectors.

Hypothesis~\ref{hyp:H1} states that the added edges in a successful evasion attack overlap with the most influential edges in a pre-attack counterfactual explanation subgraph. 

\begin{hypothesis}\label{hyp:H1}

For a target node $v$, let $\Delta G(E^+)$ denote the set of added edges in an evasion attack that flips the prediction of $f$, and let $CFEx(G)$ denote the pre-attack counterfactual explanation subgraph of the graph $G$. 
Then, there exists a high graph similarity between $\Delta G(E^+)$ and $CFEx(G)$. 
The proof is provided in the Appendix~\ref{sub:proof-H1}.
\end{hypothesis}

Building on the hypothesis, in Appendix~\ref{sec:hypotheses}, we also present two propositions and two corollaries that formalize when attack-based additions outperform deletions in flipping GNN predictions. These results characterize conditions under which deletions provably fail, yet targeted additions succeed, focusing on the functional advantage of attack-informed counterfactuals.

\section{\modelname: Methodology for Counterfactual Generation}
\label{sec:methodology}
Our objective is to design a counterfactual explainer that simultaneously incorporates \textbf{edge addition} ($E^+$) and \textbf{edge deletion} ($E^-$), combining GNN adversarial attacks with counterfactual explanation concepts. This explainer should generate high-impact perturbations while maintaining interpretability and realism. In particular, we jointly optimize three core objectives: \textbf{Impact} --- efficacy in altering model predictions; \textbf{Sparsity} --- minimal edits for interpretability; 
    \textbf{Plausibility} --- semantic validity of graph modifications.  
Figure \ref{fig:flow} illustrates the end-to-end architecture of our \modelname framework, which unifies adversarial edge perturbations with counterfactual explanation generation through a joint optimization of impact, sparsity, and plausibility.

\begin{figure}[t]
  \centering
  \includegraphics[width=1.15\textwidth]{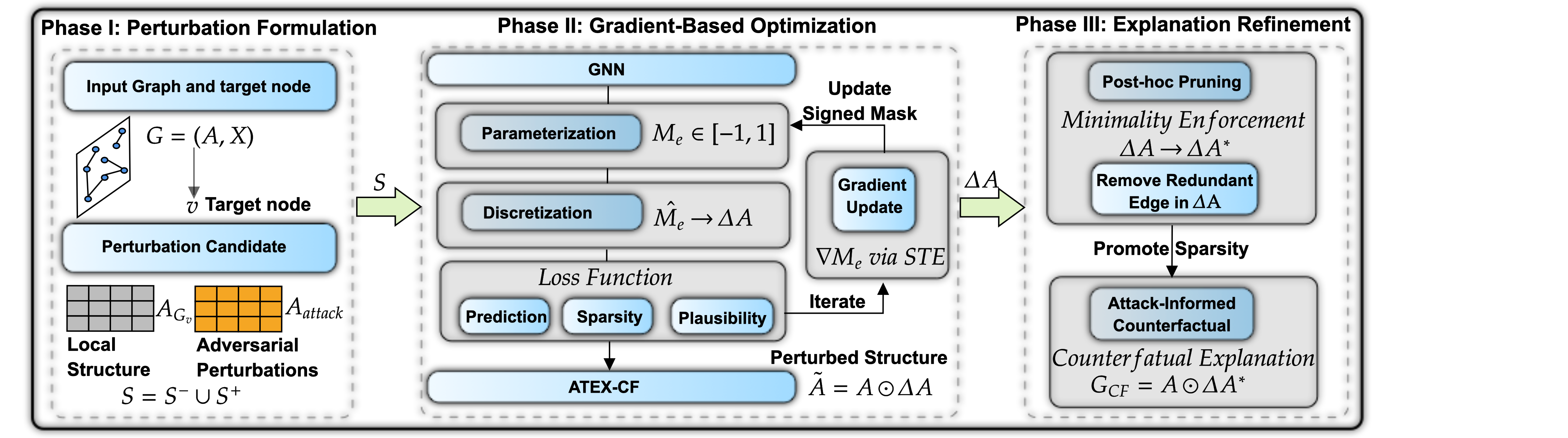}
  \caption{End-to-end workflow of the \modelname framework for counterfactual edge generation.}
  \label{fig:flow}
  \vspace{-5px}
\end{figure}

To operationalize these objectives, we cast counterfactual generation as an optimization problem over edge perturbations. 
Given a candidate set $\mathcal{S}$ of feasible edge edits, we search for $\Delta\mathbf{A}\in\mathcal{S}$ that flips the prediction of the target node while balancing sparsity and plausibility. 
This is achieved by defining a composite loss with three components, corresponding to our objectives.

\noindent\textbf{Loss Function.}  
We formulate counterfactual generation as  
\begin{equation}
\min_{\Delta\mathbf{A}\in\mathcal{S}} \; 
\mathcal{L}(\Delta\mathbf{A}) = 
\lambda_1 \mathcal{L}_{pred}(\Delta\mathbf{A})
+ \lambda_2 \mathcal{L}_{dist}(\Delta\mathbf{A})
+ \lambda_3 \mathcal{L}_{plau}(\Delta\mathbf{A}),
\label{eq:all_loss}
\end{equation}
where $\mathcal{S}$ is the candidate search space. 
Here $\mathcal{L}_{pred}$ enforces label flipping, 
$\mathcal{L}_{dist}$ penalizes the number of edge edits, 
and $\mathcal{L}_{plau}$ enforces plausibility constraints. 
Weights $\lambda_i \ge 0$ balance these terms. Next, we will define each loss function.

\noindent\textbf{Prediction Loss.}  
We denote by $f(\mathbf{A}_v, \mathbf{X}_v; \mathbf{W})$ the prediction for node $v$ under the original adjacency $\mathbf{A}_v$, and by $g(\mathbf{A}_v, \mathbf{X}_v, \mathbf{W}; \Delta\mathbf{A})$ the prediction under a perturbed adjacency $\mathbf{A}_v \odot \Delta\mathbf{A}$. Both share the same weights $\mathbf{W}$; the difference lies only in the perturbation $\Delta\mathbf{A}$.  

To encourage prediction flips, we define the loss as
\begin{multline}
\mathcal{L}_{{pred}}(\Delta\mathbf{A})
= - \, \mathbb{I}\!\left[f(\mathbf{A}_v, \mathbf{X}_v; \mathbf{W}) = f(\mathbf{A}_v \odot \Delta\mathbf{A}, \mathbf{X}_v; \mathbf{W})\right] \\
\cdot \mathcal{L}_{\text{NLL}}\!\left(f(\mathbf{A}_v, \mathbf{X}_v; \mathbf{W}),\, g(\mathbf{A}_v, \mathbf{X}_v, \mathbf{W}; \Delta\mathbf{A})\right).
\label{eq:pred-loss}
\end{multline}

The indicator ensures that the loss is active only when the perturbed graph yields the same prediction as the original. In that case, the negative log-likelihood term penalizes the perturbed prediction, pushing it away from the original class. Once a flip occurs, the loss becomes zero. Although this objective is non-differentiable due to the discrete nature of the indicator function, we employ the straight-through estimator (STE) to enable gradient-based optimization, as detailed in \S \ref{sec:perturbation}. 

\noindent\textbf{Sparsity Loss.}
\label{sec:sparsity}
To encourage concise and interpretable modifications, we impose a sparsity penalty on the number of structural edits. Specifically, we minimize the $\ell_0$ norm of the adjacency change $\Delta\mathbf{A} = \Delta\mathbf{E}^+ \cup \Delta\mathbf{E}^-,$ where $\Delta\mathbf{E}^+$ and $\Delta\mathbf{E}^-$ denote the sets of added and removed edges, respectively. $\mathcal{L}_{dist}(\Delta\mathbf{A})=\|\Delta\mathbf{A}\|_0$.
 
The objective $\mathcal{L}_{dist}(\Delta\mathbf{A})$ measures the total number of edits. By requiring $\|\Delta\mathbf{A}\|_0$ to be small, we keep the modified graph close to the original, curb unnecessary complexity, and reduce overfitting.

\noindent\textbf{Plausibility Loss.}
\label{sec:plausibility}
When generating counterfactual graphs by adding/removing edges, we must control the plausibility of the produced structure. For example, in a citation graph, an old article cannot cite a more recent article. The plausibility penalty discourages unnatural degree/motif changes: $\mathcal{L}_{plau}(\Delta\mathbf{A})=\mathcal{C}(\Delta\mathbf{A})
= \alpha_{deg} \cdot \mathrm{DegAnom}(\Delta \mathbf{A}) + \alpha_{motif} \cdot\mathrm{MotifViol}(\Delta \mathbf{A})$. We tune $\alpha_{deg}$ and $\alpha_{motif}$ to enforce realism; larger $\alpha_{deg}$ avoids implausible degree jumps, larger $\alpha_{motif}$ avoids implausible clustering jumps.

\begin{align}
\mathrm{DegAnom}(\Delta \mathbf{A})
&= \sum_{v_i\in V_{\mathrm{sub}}}\frac{\bigl|\deg_{\mathbf{\tilde A}_{v_i}}(v_i)-\deg_{\mathbf{A}_{v_i}}(v_i)\bigr|}{1+\deg_{\mathbf{A}_{v_i}}(v_i)}, \label{eq:deganom}
\\
\mathrm{MotifViol}(\Delta \mathbf{A})
&= \sum_{v_i\in V_{\mathrm{sub}}}|c_{\mathbf{\tilde A}_{v_i}}(v_i)-c_{\mathbf{A}_{v_i}}(v_i)|. \label{eq:motifviol}
\end{align}

$\mathrm{DegAnom}$ penalizes large relative changes in node degree to prevent structural anomalies, where $\deg_\mathbf{A}(v)$ and $\deg_{\mathbf{\tilde{A}}}(v)$ are degrees of node $v$ before and after modification. $\mathrm{MotifViol}$ penalizes drastic changes in local motifs, measured via clustering coefficients $c_{\mathbf{A}_v}(v)$ and $c_{\mathbf{\tilde A}_v}(v)$.

\subsection{Candidate Selection}
\label{sec:candidate-selection}

As a key aspect in \modelname, we constrain the search space of possible perturbations $\Delta\mathbf{A}$ to a pre-selected candidate set $\mathcal{S}$. This tractable set is constructed through a dual mechanism that incorporates both \textbf{local neighborhood structures} and \textbf{non-local, attack-informed candidates}, balancing interpretability with the ability to discover impactful counterfactuals.

\textbf{Edge Deletion Candidates ($\mathcal{S}^-$):} We follow the principle of \textit{actionability} and \textit{plausibility} \citep{wachter2017counterfactual}; counterfactual explanations should suggest meaningful changes within an entity's sphere of influence (e.g., local graph neighborhood), rather than involving arbitrary, distant entities. As a result, candidate edges for removal are restricted to the existing edges within the $(l+1)$-hop neighborhood $\mathcal{N}^{l+1}(v)$ of the target node $v$, i.e., $\mathcal{S}^{-} = \{e \mid e \in E, e \in \mathcal{N}^{l+1}(v)\}$.

\textbf{Edge Addition Candidates ($\mathcal{S}^{+}$):}
To overcome the limitation of deletion-only approaches and incorporate insights from adversarial attacks, our key innovation is to draw candidate edges for addition from adversarial attack subgraphs. Specifically, we employ the latest \textsc{GOttack} method \citep{alom2025gottack} to generate a set of candidate edges $\Delta A_{\text{attack}}$ for the target node $v$. \textsc{GOttack} identifies influential nodes for edge addition by learning the \textbf{graph orbit characteristics} of nodes that, when connected to $v$, maximally increase the probability of misclassification. An orbit in graph theory represents the role of a node within its local substructure (e.g., a central node in a star graph). The underlying Hypothesis~\ref{hyp:H1} of GOttack, validated by our experiments in Table \ref{tab:sim_correct}, is that nodes occupying similar structural roles (orbits) often have similar predictive influences on the target node. Therefore, edges suggested by \textsc{GOttack} (connecting $v$ to nodes in specific, influential orbits) are both highly impactful and structurally coherent.

\textbf{Final Candidate Set and Local Graph Formation:}
The complete candidate set is the union $\mathcal{S} = \mathcal{S}^{-} \cup \mathcal{S}^{+}$. The adjacency matrix $\mathbf{A}_v$ for the local subgraph used in subsequent optimization (Eq. \ref{eq:pred-loss}) is then formed by combining the original $(\ell+1)$-hop neighborhood structure of $v$ and the adversarial perturbation candidates:
\begin{equation}
\mathbf{A}_v = \underbrace{\mathbf{A}_{G_v}}_{\text{local structure}} + \underbrace{\Delta\mathbf{A}_{\text{attack}}}_{\text{adversarial perturbations}}
\label{eq:av-definition}
\end{equation}
This formulation provides a focused and principled search space $\mathcal{S}$ that is crucial for the efficiency and effectiveness of our counterfactual search algorithm. We use the $(l+1)$-hop neighborhood because an $l$-layer GCN aggregates information from nodes up to $l$ hops away; including the $(l+1)$-hop ensures that all nodes and edges within the target’s effective receptive field—including those that can indirectly influence its representation—are considered as candidates.

\subsection{Signed-Mask Perturbation and Forward Discretization}
\label{sec:perturbation}

After candidate edges are selected, the challenge is to optimize over the discrete choices of additions and deletions. Since direct optimization of binary graph structures is non-differentiable, we employ a continuous signed mask relaxation. In the forward pass, the mask is discretized into $\{-1,0,+1\}$ to yield concrete perturbations, while in backpropagation, the straight-through estimator treats this step as identity, allowing gradients to propagate through discrete edge decisions. This process is carried out as follows and the complete \modelname framework is given in Alg.~\ref{alg:hybrid-cf}.

Each candidate edge $e \in \mathcal{S}$ (where $\mathcal{S}$ is the candidate set defined in \S \ref{sec:candidate-selection}) is associated with a continuous signed parameter $M_e \in [-1,1]$ \rwOne{(Line ~\ref{alg:line:init} of Alg.~\ref{alg:hybrid-cf})}. This parameter encodes both the directionality and the magnitude of the proposed modification; a signed mask variable $M_e$ encodes perturbations, with $M_e>0$ denoting an edge addition ($e\in\Delta\mathbf{E}^+$), $M_e<0$ denoting an edge deletion ($e\in\Delta\mathbf{E}^-$), and $M_e\approx0$ no modification. Here, the sign of $M_e$ indicates the type of operation (addition or deletion), while the magnitude $|M_e|$ reflects the proposed strength or importance of the perturbation. This continuous representation facilitates gradient-based learning.

During the forward pass, we discretize these continuous parameters to obtain a binary perturbation matrix. This process involves two steps: thresholding and sparsity enforcement. First, we apply thresholding to convert $M_e$ into a ternary value. The discretized mask is obtained by thresholding, $\widehat{M}_e = +1$ if $M_e > \tau^+$, $\widehat{M}_e = -1$ if $M_e < -\tau^-$, and $\widehat{M}_e = 0$ otherwise \rwOne{(Line ~\ref{alg:line:disc} of Alg.~\ref{alg:hybrid-cf})}.

where $\tau^+$ and $\tau^-$ are positive thresholds that control the sensitivity for edge addition and deletion, respectively. Typically, we set $\tau^+ = \tau^- = 0.5$ to ensure symmetry.

\begin{wrapfigure}{R}{0.55\linewidth}
\vspace{-1em}
\begin{minipage}{\linewidth}
\hrule height0.8pt \kern2pt
\captionof{algorithm}{\modelname: Counterfactual Generator\vspace{-9px}}
\label{alg:hybrid-cf}
\hrule height0.8pt \kern2pt
\begin{algorithmic}[1]
  \Require Graph $G=(\mathbf{A},X)$, model $f$, target node $v$, candidate set $\mathcal{S}$
  \State Initialize mask $M_e \gets \mathbf{0}$ over $\mathcal{S}$ \label{alg:line:init}
  \For{$t=1$ to $T_{\max}$}
    \State $\widehat{M}_e \gets \textsc{Threshold}(M_e, \tau^+, \tau^-)$ \Comment{Discretize} \label{alg:line:disc}
    \State $\Delta\mathbf{A} \gets \textsc{Top-}\kappa(|M_e|)$ \Comment{Sparsify} \label{alg:line:spar}
    \State Evaluate $\mathcal{L}(M_e)$ on $\mathbf{A} \odot \Delta\mathbf{A}$ \label{alg:line:loss}
    \State $M \gets M - \eta \nabla_M \mathcal{L}(M)$ \Comment{Update via STE} \label{alg:line:ste}
    \If{flipped($v$) \textbf{and} $\|\Delta\mathbf{A}\|_0$ stable} \label{alg:line:end}
       \State \textbf{break}
    \EndIf
  \EndFor
  \State \Return \Call{Prune}{$\Delta\mathbf{A}, G, f, v$} \Comment{See Alg.~\ref{alg:prune}}
\end{algorithmic}
\kern2pt\hrule height0.8pt
\end{minipage}
\vspace{-2.5em}
\end{wrapfigure}

Next, to enforce the perturbation budget constraint $\|\Delta\mathbf{A}\|_0 \leq \kappa$, we retain only the $\kappa$ edges with the largest magnitudes $|M_e|$ and assign their discretized values $\widehat{M}_e \in \{-1,0,+1\}$ to the corresponding entries in the adjacency matrix \rwOne{(Line ~\ref{alg:line:spar} of Alg.~\ref{alg:hybrid-cf})}. The perturbation matrix is defined as $\Delta\mathbf{A}_{i,j}=\widehat{M}_e$ if edge $(i,j)$ is among the top-$\kappa$ candidates ranked by $|M_e|$, and $0$ otherwise. This ensures that at most $\kappa$ edges are modified, producing sparse and interpretable counterfactuals.

The objective loss and resulting perturbed adjacency matrix is then computed as 
$\widetilde{\mathbf{A}} = \mathbf{A} \odot \Delta\mathbf{A}$, 
where the operator $\odot$ applies the signed edge modifications encoded in $\widehat{M}_e \in \{-1,0,+1\}$ \rwOne{(Line ~\ref{alg:line:loss} of Alg.~\ref{alg:hybrid-cf})}. To maintain differentiability through this discretization step, we employ the straight-through gradient estimator (STE) during backpropagation $\frac{\partial \widehat{M}_e}{\partial M_e} \approx 1$. This approximation \rwOne{(Line~\ref{alg:line:ste} of Alg.~\ref{alg:hybrid-cf})} allows gradients to flow directly through the binarization operation, treating the discretization as if it were an identity function in the backward pass \citep{bengio2013estimating}. Consequently, the continuous parameters $M_e$ can be updated using gradient descent, even though the forward pass involves non-differentiable operations. This approach is widely used in training binary neural networks and has been shown to be effective in practice. The loop stops early if the target node $v$ flips and the perturbed edges $\|\Delta\mathbf{A}\|_0$ are stable \rwOne{(Line~\ref{alg:line:end} of Alg.~\ref{alg:hybrid-cf})}.

\noindent\textbf{Minimality-Aware Post-Hoc Pruning}
\label{sec:minimality}
While the training loss promotes sparsity and plausibility in expectation, the discrete relaxation can leave redundant edges active in $\Delta\mathbf{A}$.  This occurs mostly due to noisy or approximate gradient updates that over-compensate. 
To enforce the minimality of counterfactual explanations, we adopt a simple yet effective greedy algorithm \rwOne{(Alg.~\ref{alg:prune} in Appendix \ref{sec:alg2})}. Edges in the candidate set are ranked by their importance score $\psi_e \propto |\partial\mathcal{L}/\partial M_e|$ \rwOne{(approximated gradient magnitude at line~\ref{alg:line:rank} of Alg.~\ref{alg:prune})}. The algorithm then iteratively removes the least important edge, checking if the prediction flip persists \rwOne{(Line~\ref{alg:line:remove} of Alg.~\ref{alg:prune})}. This continues until no more edges can be removed without reverting the prediction, and it attains final perturbation $\Delta\mathbf{A}^*$.

\section{Experiments}

\subsection{Experimental Setup}
\label{sec:experiments_setup}

\begin{table}[t]
\centering
\caption{Dataset statistics.}
\label{tab:datasets}
\scriptsize
\begin{tabular}{lcccccc}
\toprule
\textbf{Dataset} & \textbf{Homophily Ratio} & \textbf{\#Nodes} & \textbf{\#Edges} & \textbf{\#Features} & \textbf{\#Classes} & \textbf{Type} \\
\midrule
BA-SHAPES \citep{ying2019gnnexplainer} & 0.80 & 700 & $3958$ & -- & 4 & Synthetic   \\
TREE-CYCLES \citep{ying2019gnnexplainer} & 0.90 & 871 &1,940 & -- & 2 & Synthetic   \\
Loan-Decision \citep{ma2025c2explainer} & 0.47 & 1000 & $3950$ & 2 & 2 & Synthetic     \\
Cora \citep{sen2008collective} & 0.81 & 2,708 & 5,429 & 1,433 & 7 & Real   \\
\rwTwo{Chameleon \citep{pei2020geom}} & 0.24 & 2,277 & 36,101 & 2,325 & 5 & Real   \\
Ogbn-Arxiv \citep{hu2020open} & 0.66 & 169,343 & 1,166,243 & 128 & 40 & Real   \\
\bottomrule
\end{tabular}
\end{table}

\textbf{Datasets.} 
We evaluate \modelname on both synthetic and real-world benchmarks. 
Synthetic datasets include \textbf{BA-SHAPES} and \textbf{TREE-CYCLES} \citep{ying2019gnnexplainer}, widely used in GNN explainability, and the \textbf{Loan-Decision} social graph \citep{ma2025c2explainer}. For real-world evaluation, we use the \textbf{Cora} citation network \citep{sen2008collective} and the large-scale \textbf{ogbn-arxiv} dataset from OGB \citep{hu2020open}. \rwTwo{Additionally, we include the heterophilic \textbf{Chameleon} dataset \citep{rozemberczki2021multi}, which is known for its low feature homophily and non-community structure, providing a challenging real-world setting for counterfactual explanations.}

\textbf{GNNs.} We evaluate our approach on three standard GNN architectures: \textbf{GCN} \citep{kipf2017semi}, \textbf{GAT} \citep{velickovic2018graph}, and \textbf{Graph Transformer} \citep{shi2020masked}.

\textbf{Baselines}: We compare our method against a comprehensive set of baseline approaches, which we categorize into two groups. The first group comprises \textbf{explanation-based baselines}: \textbf{CF-GNNExplainer} \citep{lucic2022cf}, \rwTwo{\textbf{CF$^2$} \citep{tan2022learning} and \textbf{NSEG} \citep{cai2025probability}}, counterfactual methods that optimize for edge deletions using a perturbation mask; \rwTwo{\textbf{INDUCE} \citep{verma2024induce}, an inductive counterfactual framework that learns structural interventions through local subgraph rewiring, and can naturally realize both edge deletions and edge additions; 
\textbf{C2Explainer} \citep{ma2025c2explainer}, a customizable mask-based counterfactual explainer that jointly optimizes edge and feature perturbations under flexible constraints, supporting both removal and insertion of edges.}
\textbf{GNNExplainer} \citep{ying2019gnnexplainer}, a factual explainer adapted for counterfactual analysis by removing edges in descending order of importance until prediction flips; and \textbf{PGExplainer} \citep{luo2020parameterized}, another factual method adapted similarly to GNNExplainer. The second group consists of \textbf{attack-based baselines} repurposed for counterfactual generation: \textbf{Nettack} \citep{zugner2018adversarial}, a white-box adversarial attack method adapted by using its edge perturbation capability such that the target class is different from the original prediction; and \textbf{GOttack} \citep{alom2025gottack}, a recent adversarial method that leverages graph orbital theory to identify critical nodes for edge additions, making it naturally suited for generating addition-based counterfactuals. For fair comparison, all methods are constrained to a default perturbation budget (i.e., maximum possible number of edge flips) of $\kappa = 5$ edges. We vary $\kappa$ for ablation study in Figure \ref{fig:five_plots}. Explanation-based methods (CF-GNNExplainer, GNNExplainer, PGExplainer) are restricted to edge deletions only, while attack-based methods (Nettack, GOttack) and our \modelname can use both edge additions and deletions within the same budget.

In our experiments, we set the random seed (102, 103, 104) for reproducibility. For the attack model, we employed evasion attacks using the GOttack method. For the \modelname, we used a learning rate of 0.001, trained for 200 epochs, and adopted the SGD optimizer to generate counterfactual explanation with a maximal perturbed budget of 5 edges. The default loss weights were configured as follows: $\lambda_1 = 1.5$, $\lambda_2 = 0.5$, $\lambda_3 = 0.5$, $\alpha_{deg}=1.5$, and $\alpha_{motif}=1.0$. These hyperparameters were chosen to balance prediction flipping, sparsity, and plausibility in counterfactual generation.

\textbf{Evaluation Metrics}: We evaluate the performance of counterfactual explainers in misclassification rate, fidelity, explanation size, plausibility, and time costs. Definitions are given in Appendix~\ref{sec:metrics}.

\begin{table}[t]
\caption{\rwOne{Meta Results. Average ranks ($\downarrow$) across six datasets (lower is better). Ranks are computed per metric per dataset (best=1; ties get the same rank), then averaged across datasets equally. “Wins” counts how many times a method achieved rank one across all metric–dataset cells (6 datasets $\times$5 metrics = 30 cells, ties allowed).}}
\label{tab:meta}
\centering
\scriptsize
\renewcommand{\arraystretch}{1.1}
\begin{tabular}{lccccc|cc}
\toprule
\textbf{Method} & \textbf{Misclass.} & \textbf{Fidelity} & \textbf{$\Delta\mathbf{E}$} & \textbf{Plausibility} & \textbf{Time (sec)} & \textbf{Overall Avg.} & \textbf{Wins} \\
\midrule
CF-GNNExplainer & 4.7 & 4.8 & 2.0 & 2.3 & 9.5 & 4.67 & 1 \\
\rwTwo{INDUCE}          & 6.3 & 6.8 & 4.5 & 7.8 & 2.8 & 5.67 & 1 \\
\rwTwo{C2Explainer}     & 5.0 & 6.2 & 5.7 & 5.8 & 8.7 & 6.27 & 1 \\
\rwTwo{CF$^2$}     & 7.0 & 7.0 & 6.7 & 5.7 & 7.8 & 6.83 & 0 \\
\rwTwo{NSEG}     & 6.7 & 6.7 & 7.5 & 4.3 & 6.7 & 6.37 & 0 \\
GNNExplainer    & 6.8 & 7.5 & 4.7 & 5.5 & 3.3 & 5.57 & 1 \\
PGExplainer     & 8.2 & 7.7 & 4.2 & 5.8 & \textbf{1.0} & 5.37 & 6 \\
Nettack         & 3.3 & 2.5 & 8.8 & 8.0 & 4.5 & 5.43 & 2 \\
GOttack         & 4.8 & 4.3 & 8.8 & 8.0 & 3.0 & 5.80 & 0 \\
\modelname (ours)  & \textbf{1.2} & \textbf{1.3} & \textbf{1.0} & \textbf{1.2} & 7.3 & \textbf{2.40} & \textbf{20} \\
\bottomrule
\end{tabular}
\vspace{-0.5em}
\end{table}

\subsection{Results and Analysis}

We evaluate \modelname against all baselines under the same budget constraints ($\kappa =\{1,\ldots,5\}$). Table~\ref{tab:meta} summarizes average rankings across datasets and metrics. Our method achieves the best overall rank (2.40 vs. 4.67 for the next best) and wins 20/30 metric–dataset combinations, far exceeding competitors. This confirms that \modelname consistently finds more effective counterfactuals. In particular, counterfactual explainers (CF-GNNExplainer, CF$^2$, NSEG, GNNExplainer, PGExplainer) are limited to edge deletions, while counterfactual explainers (INDUCE, C2Explainer), adversarial attack methods (Nettack, GOttack) and \modelname can also add edges. Crucially, CF-GNNExplainer explicitly seeks minimal edge deletions, and GOttack systematically manipulates graph orbits to induce errors, yet neither matches \modelname on combined effectiveness and realism. \textbf{Our empirical results on individual datasets and with other GNNs (GAT and Graph Transformer) are given in Appendix~\ref{sec:ind_data}-\ref{sec:other_gnn}.}

\begin{figure*}[t]
\centering
\includegraphics[scale=0.19]{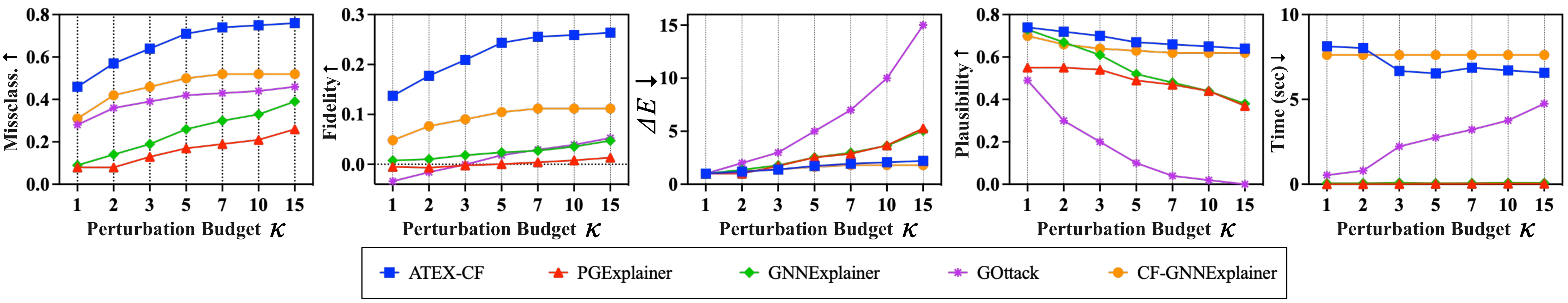}
\caption{Counterfactual explanations on \textbf{Cora} and GCN under varying perturbation budgets $\kappa$}
\label{fig:five_plots}
% \vspace{-10pt}
\end{figure*}

Figure \ref{fig:five_plots} plots performance vs. perturbation budget on Cora. As the budget grows, all methods improve: \modelname quickly raises misclassification (e.g., from 0.46 at $\kappa$=1 to 0.76 at $\kappa$=15) far above others, and maintains the highest fidelity and plausibility. Notably, \modelname’s edit size increases only mildly with $\kappa$, whereas attack baselines must exhaust all allowed edits ($\Delta E$→5). This trend illustrates that our objective effectively exploits additional budget to find better counterfactuals without excessive edits.

\noindent\textbf{Ablation Study.} Table \ref{ablation-study} in Appendix \ref{sec:ablation} shows the effect of removing each loss. Our findings demonstrate that $\mathcal{L}_{dist}$ enforces concise edits, $\mathcal{L}_{plau}$ preserves semantic plausibility, and their combination in \modelname achieves the best overall balance across all metrics.

\noindent\textbf{Sensitivity Analysis.} We next analyze key hyperparameters. \textbf{Search depth ($l$):} Figure \ref{sensitive to layer} in Appendix~\ref{sec:sensitivity} shows that $l$ = 2 captures sufficient local structure surrounding the target node for effective counterfactuals. \textbf{Hyperparameters ($\alpha_{deg}, \alpha_{motif}$):} Figure \ref{hyper_sensitive} in Appendix \ref{sec:sensitivity} demonstrates that \modelname is robust across a range of hyperparameter values (e.g., $\alpha=0.5$–$1.5$); while moderate $\alpha$ maximizes fidelity and plausibility together. 

\noindent\textbf{Impact of Pruning Strategy.} 
We also evaluate the impact of our candidate-edge pruning strategy \rwOne{(Algorithm~\ref{alg:prune} in Appendix \ref{sec:alg2})}. As shown in Figure~\ref{effec_pruning}, pruning yields more concise explanations by reducing redundant edge edits ($\Delta \mathbf{A}=1.71 \to 1.62$), but as intended, its real utility is the reduced runtime (6.12s $\to$ 3.00s), while preserving predictive accuracy (misclass.=$0.71$), plausibility ($0.76$ vs.~$0.75$), making it an effective and efficient enhancement of our framework.

\subsection{Asymmetric Costs of Edge Perturbations}
\label{sec:asym_main}

\rwFour{\paragraph{Feasibility of Edits.} While our framework assumes that edge additions and deletions are possible, not all structural edits are equally realistic in every domain. For example, in citation networks, adding an edge between two papers may be implausible after publication. To handle this, we incorporate a real-valued cost metric into our optimization that allows domain-specific constraints on edit feasibility. }

\rwOne{The default perturbation budget in ATEX-CF assumes symmetric costs for edge additions and deletions, treating both equally under the constraint $\|\Delta \mathbf{A}\|_0 = \|\Delta E^+\| + \|\Delta E^-\| \leq \kappa$. To account for scenarios where additions may be less actionable or more expensive (e.g., in real-world graphs with immutable structures), we introduce a scalar weight $C > 0$ to control the cost asymmetry between perturbation types. The new constraint becomes:}
$$
C \cdot \|\Delta E^+\| + \|\Delta E^-\| \leq \kappa.
$$

\rwOne{We vary $C$ over a range of values (from 0.5 to 21.0) to analyze how penalizing edge additions affects counterfactual quality. Table~\ref{tab:asym20} reports results on Cora using a GCN under budget setting: $\kappa = 20$. More results under different $\kappa$ are provided in Appendix~\ref{sec:asym_appendix}.}

\begin{table}[ht]
\centering
\caption{\rwOne{Counterfactual performance under asymmetric addition cost $C$ with $\kappa=20$.}}
\label{tab:asym20}
\resizebox{\textwidth}{!}{
\begin{tabular}{ccccccc}
\toprule
Addition Cost & Deletion Cost & Misclass. & Fidelity & $\Delta E$ (E$^+$, E$^-$) & Plausibility & Time (sec) \\
\midrule
0.5 & 1.0 & 0.70 & 0.23 & 1.78 (0.78, 1.00) & 0.72 & 6.1 \\
0.8 & 1.0 & 0.70 & 0.23 & 1.78 (0.77, 1.01) & 0.71 & 6.1 \\
1.0 & 1.0 & 0.70 & 0.23 & 1.78 (0.77, 1.01) & 0.71 & 10.3 \\
3.0 & 1.0 & 0.70 & 0.23 & 1.82 (0.66, 1.16) & 0.69 & 10.7 \\
5.0 & 1.0 & 0.70 & 0.23 & 1.82 (0.65, 1.17) & 0.69 & 10.9 \\
10  & 1.0 & 0.69 & 0.23 & 1.80 (0.61, 1.19) & 0.69 & 11.0 \\
15  & 1.0 & 0.69 & 0.23 & 1.76 (0.60, 1.16) & 0.69 & 11.0 \\
20  & 1.0 & 0.54 & 0.15 & 1.42 (0.49, 0.93) & 0.68 & 11.5 \\
21  & 1.0 & 0.42 & 0.10 & 1.78 (0.00, 1.78) & 0.62 & 11.7 \\
\bottomrule
\end{tabular}}
\end{table}

\rwOne{As the addition cost $C$ increases relative to the deletion cost, the optimizer favors deletion-heavy or deletion-only solutions. This leads to a trade-off: higher $C$ reduces misclassification success and fidelity, and in extreme cases ($C \geq 20$), restricts exploration to only deletions, thereby weakening plausibility. These results confirm that asymmetric weighting serves as an effective control knob for tailoring explanation strategies in cost-sensitive domains, but also warn against overly skewed values that impair explanation quality.}

\section{Conclusions}
We presented \modelname, a theoretically grounded framework that unifies adversarial attacks and counterfactual explanations for graph neural networks. By incorporating both edge additions and deletions under a constrained budget, \modelname generates explanations that are not only faithful but also informative. Our joint optimization of fidelity, sparsity, and plausibility ensures instance-level counterfactuals that balance interpretability with realism. Experiments on synthetic and real-world benchmarks confirm the effectiveness of this integration, highlighting how adversarial insights can substantially improve the quality of counterfactual explanations, compared with state-of-the-art counterfactual generation and adversarial attack methods. \rwFour{Future directions include extending \modelname to dynamic graphs, incorporating node feature perturbations, and applying the framework to real-world domains like healthcare and finance.}

%\Yu{The additions and deletions are assumed to be equally feasible, asymmetric costs are not explored. Can the framework handle domains where additions are infeasible or more expensive than deletions? Have you tried asymmetric action costs?}

%\Yu{Moreover, although the authors have considered using plausibility loss to prevent generating out-of-distribution explanations, the OOD problem is actually very domain-specific. Avoiding implausible degree jumps and implausible clustering jumps (namely these techniques used in the paper) can only slightly mitigate these issues. (Actions: I would like to see more discussions regarding this and maybe other solutions to avoid generating OOD explanations). --While degree/clustering penalties are coarse, they serve as a general proxy, practical application can incorporate more strict limitation to avoid out-of-distribution problem.}

%\Yu{Scope: Only node classification; no feature-perturbation counterfactuals; limited real-world user studies for “plausibility.” -- discuss future extensions to feature perturbations and graph classification}

%\Yu{What is the future with this framework or research, what are some future directions to take this or what further implications can be gleaned from this work?  --Future directions include extending ATEX-CF to dynamic graphs, incorporating node feature perturbations, and applying the framework to real-world domains like healthcare and finance.}

\bibliography{CounterAttack}
\bibliographystyle{iclr2026_conference}

\newpage
\appendix
\section*{Reproducibility Statement}
We provide the full implementation of our models and experimental setup to ensure reproducibility. Experimental results are reported as the mean and standard deviation across different random seeds, and the hyperparameters used are detailed in Section~\ref{sec:experiments_setup}. 
\textbf{Our code and data are available at \nolinkurl{https://github.com/zhangyuo/ATEX_CF}.}

\section{Appendix}

\subsection{Limitation}
A key limitation of this study is the assumption that edge additions and deletions are equally feasible, which may not hold in domains where graph modifications are inherently constrained. Future work could incorporate domain-specific constraints and node-feature perturbations to enhance the practical relevance of \modelname while preserving its theoretical contributions. The central premise of our approach is that unifying adversarial attack strategies with counterfactual reasoning strengthens both the fidelity and plausibility of explanations. Unlike methods that treat these perspectives independently, \modelname provides a principled integration that balances model sensitivity with explanation realism in a computationally tractable way.

\rwThree{Our current framework does not consider poisoning attacks or node feature perturbations. Poisoning requires retraining the model after each perturbation, which conflicts with our fixed-model counterfactual setup. Feature perturbations, while useful in some domains, are harder to constrain plausibly and often lack structural interpretability in graph settings. Extending ATEX-CF to incorporate these forms of attacks is a promising direction for future work.}

\rwTwo{We also recognize that the current plausibility loss, based on structural proxies such as degree anomaly and motif violations, may only partially address the risk of generating out-of-distribution explanations. These generic regularizers do not account for domain-specific constraints such as temporal consistency or semantic incompatibility (e.g., citing future papers or linking unrelated functional modules). To mitigate this, we emphasize that our pruning step (Algorithm~\ref{alg:prune}) serves as a late filter that discards non-essential edits, including those that may be structurally valid yet semantically implausible. In practice, pruning significantly reduces the number of active perturbations, helping ensure that final explanations are not only minimal but more likely to remain within the distribution.}

\subsection{Related Work}

\rwOne{We begin a complete account of counterfactual explanations and adversarial examples with their philosophical origins, where both have been studied as forms of contrastive reasoning and causal dependence.}

\rwOne{\citet{freiesleben2022intriguing} examines the connection between counterfactual explanations and adversarial examples in standard machine learning, studying their shared optimization goal and conceptual and historical development in fields such as philosophy and psychology. The analysis remains focused on models with independent input features, where perturbations do not propagate through structured dependencies. Our work builds on this by studying the connection in graph neural networks, where node predictions depend on message passing over edges. In this setting, identifying minimal changes that flip predictions becomes harder, and naive perturbations can easily break plausibility. We address these challenges by using adversarial edge additions to guide counterfactual generation and show that attack-informed edits offer an effective and realistic way to produce explanations in graphs.}

\paragraph{GNN Explanations.}
Different categories of GNN explanation methods have been developed to offer diverse perspectives and improve the interpretability of GNN models \citep{DBLP:conf/dsaa/KhanM23,yuan2022explainability}. 
Two main categories of explanations persist: factual and counterfactual.
\textbf{Counterfactual explanations}, which are the focus of this work, provide explanations by identifying the minimum perturbation or change to the input graph that leads to a different prediction from the model \citep{bajaj2021robust, huang2023global, tan2022learning}, thereby revealing the most critical structures underlying the decision. Existing methods are predominantly based on edge deletions. For instance, {\sf CF-GNNExplainer} \citep{lucic2022cf}, {\sf RCExplainer} \citep{bajaj2021robust}, {\sf GNN-MOExp} \citep{dandl2020multi}, {\sf CF$^2$} \citep{tan2022learning}, {\sf NSEG} \citep{cai2025probability}, {\sf Banzhaf} \citep{chhablani2024game}, and {\sf CF-GFNExplainer} \citep{he2024learning} all design deletion-oriented mechanisms, such as gradient-based mask optimization, decision boundary constraints, multi-objective optimization, or probabilistic sampling. These approaches emphasize faithfulness, sparsity, or necessity/sufficiency guarantees, but rely mainly on removing salient substructures. 

More recently, several works on node classification have extended counterfactual explanations to include edge additions, or the joint use of both addition and deletion. {\sf INDUCE} \citep{verma2024induce} treats counterfactual search as a Markov decision process, allowing the model to learn edge modifications (both additions and deletions) that lead to flips. {\sf C2Explainer} \citep{ma2025c2explainer} further integrates hypergraph representations with straight-through optimization to balance reliability and fidelity, and explicitly models the potential risks of false evidence from edge additions. In the context of graph classification, approaches such as counterfactual graphs \citep{abrate2021counterfactual}, {\sf CLEAR} \citep{ma2022clear}, {\sf GCFExplainer} \citep{kosan2024gcfexplainer}, and density-based counterfactual graphs \citep{abrate2023counterfactual} adopt generative or global search strategies that combine edge addition and deletion to ensure causally consistent and semantically coherent explanations. 

Overall, while edge-deletion-based methods dominate current counterfactual explanation research, the emerging edge-addition or mixed approaches demonstrate that edge addition can serve as a complementary mechanism, especially in cases where deletion-based explanations fail to capture counterfactual reasoning. This motivates our design of hybrid counterfactual explainers that leverage both deletion and addition.
Unlike prior counterfactual methods that may include edge additions, our approach is the first to integrate adversarial attack strategies—systematically leveraging their capacity to identify high-impact edge additions—with traditional deletion-based reasoning, thereby unifying two separately studied domains to generate more effective and actionable explanations. Table~\ref{apx-1} summarizes important GNN explanation methods, including both factual and counterfactual approaches, along with their explanation type, candidate modification, and target task.

\begin{table*}[h]
\caption{Characteristics important GNN explainers including ours. “E , F, N” denote removing/adding edges, node feature modification, removing/adding nodes, respectively. GC and NC denote graph classification and node classification, respectively.}
\label{apx-1}
\centering
\small
\renewcommand{\arraystretch}{1.0}
\begin{tabular}{cccc}
\toprule
{\bf Method} & {\bf Type} & {\bf Candidate} & {\bf Task} \\
\midrule
GNNExplainer \citep{ying2019gnnexplainer} & factual/instance-level & E, N & GC/NC \\
 
PGExplainer \citep{luo2020parameterized} & factual/instance-level & E & GC/NC \\
 
MOO \citep{liu2021multi} & counterfactual/instance-level & E(-), N(-) & NC \\
 
CF-GNNExplainer \citep{lucic2022cf} & counterfactual/instance-level & E(-) & NC \\
 
RCExplainer \citep{bajaj2021robust} & counterfactual/instance-level & E(-) & GC/NC \\
 
CF$^{2}$ \citep{tan2022learning} & counterfactual/instance-level & E(-), F & GC/NC \\
 
INDUCE \citep{verma2024induce} & counterfactual/instance-level & E(+,-) & NC \\
 
NSEG \citep{cai2025probability} & counterfactual/instance-level & E(-), F & GC/NC \\
 
Banzhaf \citep{chhablani2024game} & counterfactual/instance-level & E(-) & NC \\
 
C2Explainer \citep{ma2025c2explainer} & counterfactual/instance-level & E(+,-), F & GC/NC \\
 
\modelname (ours) & counterfactual/instance-level & E(+,-) & NC \\
\bottomrule
\end{tabular}
\end{table*}

\paragraph{GNN Adversarial Attacks.}
Graph adversarial attacks investigate structural perturbations but from a different perspective: their objective is to reduce model performance rather than to improve interpretability. These attacks can be divided into two main categories: evasion attacks and poisoning attacks \citep{yuan2022explainability,DBLP:journals/csur/LongaASCLLP25}.
In evasion attacks, the GNN parameters are fixed and the adversary perturbs the test graph to flip predictions without retraining. Examples include targeted edge modifications during inference \citep{zou2021tdgia, chang2020restricted, ma2020towards, fan2023jointly}. Poisoning attacks, in contrast, manipulate the training data by injecting adversarial samples, forcing the retrained model to internalize the perturbations and degrade performance \citep{alom2025gottack, zugner2018adversarial, li2021adversarial, chen2018fast, geisler2021robustness}.

Empirical studies show that adversarial evasion attacks on GNNs --- particularly those based on strategically adding edges --- exploit data biases and model weaknesses to induce misclassifications, in stark contrast to counterfactual explanations, which predominantly rely on edge deletions. Integrating these attack‐inspired edge‐addition perturbations into counterfactual frameworks can enrich explanation graphs and forge a novel link between adversarial robustness and interpretability.

\paragraph{Pitfalls in Explanations}

\rwTwo{Faber et al.~\cite{faber2021comparing} identify five pitfalls that affect the evaluation of GNN explanation methods: (1) the GNN may rely on bias terms or spurious features rather than the annotated evidence, (2) the ground-truth explanation may be redundant or non-unique, leading to mismatches during scoring, (3) some datasets allow trivial explanations, such as those based on nearest neighbors or centrality, (4) weak models that do not learn the true structure render all explanation assessments unreliable, and (5) explanation behavior may vary significantly across architectures even with similar accuracy. These concerns are valid for attribution-based explanations evaluated against fixed motifs, but they do not apply to our setting. Our method generates counterfactuals by identifying sparse edge perturbations that flip the model’s prediction. We do not assume or require ground-truth substructures, nor do we compare explanations to predefined motifs. Instead, our evaluation reflects the model’s actual behavior and is based on plausibility and edit sparsity. This makes our approach robust to the ground-truth mismatches discussed in~\cite{faber2021comparing, agarwal2023evaluating}.}

\paragraph{Fusing GNN Explanations and Robustness against Attacks.}

Recent efforts on robust explainable graph neural networks combine explainability with adversarial defense to preserve explanation quality under worst-case perturbations. GNNEF \citep{DBLP:conf/icml/LiPD0W24} reveals that perturbation-based explainers (e.g., GNNExplainer, PGExplainer) are highly fragile, as minor structural changes can drastically alter explanations without affecting predictions, and proposes loss- and deduction-based attacks exposing this vulnerability across both graph- and node/edge-level tasks. \citet{fan2023jointly} develop GEAttack that can attack both a GNN model and its explanations
by simultaneously exploiting their vulnerabilities. \citet{ChandaGS25} exploit explainability-based strategy to devise adversarial attacks on GNNs. Complementarily,  \citet{DBLP:journals/corr/abs-2505-02566} introduce a benchmark analyzing the interplay between robustness and interpretability under poisoning and evasion attacks, showing that most defenses improve interpretability but with architecture-dependent trade-offs and limitations in existing metrics. Building on these insights, XGNNCert \citep{DBLP:conf/iclr/LiPD0W25} provides the first certifiable robustness guarantee for graph-level tasks, ensuring stable explanations without sacrificing predictive performance. At the node/edge level, $k$-RCW \citep{DBLP:conf/icde/QiuW0W24} proposes robust counterfactual witnesses (RCWs) that remain factual, counterfactual, and resilient to structural disturbances, while GNNNIDS \citep{DBLP:conf/itasec/GalliVMM25} introduces an evaluation framework for intrusion detection via structural adversarial attacks, demonstrating that Integrated Gradients produces precise yet exploitable explanations. While these works improve the robustness of explanations, to the best of our knowledge, we are the first to unify adversarial attack techniques such as both edge additions and deletions for better counterfactual explanation generation.

\subsection{Summary of notations used in this paper}

Table~\ref{tab:compact_notations} provides a concise summary of the key notations used in this paper, covering graph structure, node features, GNN models, optimization terms, and theoretical concepts.

\begin{table}[h]
\centering
\caption{Summary of notations used in this paper}
\label{tab:compact_notations}
\begin{tabular}{@{}lL@{}}
\toprule
\textbf{Symbol Group} & \textbf{Description} \\
\midrule
\multicolumn{2}{c}{\textbf{Graph Structure}} \\
$G, V, E$ & Input graph, node set, edge set \\ 
$N, m$ & Number of nodes and edges ($N = |V|$, $m = |E|$) \\
$\mathbf{A}, \mathbf{A}_{self}$ & Adjacency matrix $\mathbf{A}\in\{0,1\}^{n\times n}$, adjacency matrix with self-loops ($\mathbf{A} + \mathbf{I}_N$) \\
$\mathbf{D}, \hat{\mathbf{A}}$ & Degree matrix, normalized adjacency matrix ($\mathbf{D}^{-\frac{1}{2}}\mathbf{A}_{self}\mathbf{D}^{-\frac{1}{2}}$) \\
$\widetilde{\mathbf{A}}, \Delta\mathbf{A}$ & Perturbed adjacency matrix ($\mathbf{A} \odot \Delta\mathbf{A}$), edge modifications ($\in\{-1,0,1\}^{n\times n}$) \\
$\Delta\mathbf{E}^+, \Delta\mathbf{E}^-$ & Added/deleted edge sets \\
\midrule
\multicolumn{2}{c}{\textbf{Node Features, Neighborhood, \& Labels}} \\
$\mathcal{N}^{l}(v)$ & $l$-hop neighborhood of node $v$ \\
$\mathbf{X}$ & Node feature matrix ($\in\mathbb{R}^{n\times d}$) \\
$v, y_v, \hat{y}_v$ & Target node, ground-truth label, predicted label \\
\midrule
\multicolumn{2}{c}{\textbf{GNN Model}} \\
$\mathbf{W}^{(l)}, \mathbf{H}^{(l)}$ & Weight matrix and hidden representations at GNN layer $l$ \\
$\mathbf{Z}$ & Output logits \\
$f(\mathbf{A}, \mathbf{X}, v)$ & GNN prediction for node $v$  \\
\midrule
\multicolumn{2}{c}{\textbf{Optimization \& Loss}} \\
$\mathcal{L}(\bullet)$ & Loss objective function \\
$\mathcal{L}_{{pred}}, \mathcal{L}_{{dist}}$ & Prediction loss (flipping), sparsity loss (minimal edits) \\
$\mathcal{L}_{{plau}}, \mathcal{C}(\Delta\mathbf{A})$ & Plausibility loss, plausibility penalty \\
$\|\Delta\mathbf{A}\|_0, \kappa$ & Number of changed edges, perturbation budget \\
$M_e, \widehat{M}_e$ & Continuous signed mask ($\in[-1,1]$), discretized mask ($\in\{-1,0,1\}$) \\
$\tau^+, \tau^-$ & Positive/negative thresholds for discretization \\
$\psi_e, \mathcal{S}$ & Edge importance score, candidate modification set \\
$\lambda_1,\lambda_2,\lambda_3$ & Loss trade-off weights \\
$\alpha_{deg}, \alpha_{motif}$ & Realism penalty weights \\
$\eta$ & Learning rate \\
\midrule
\multicolumn{2}{c}{\textbf{Theoretical Concepts}} \\
$m_v$ & Prediction margin \\
$g_e$ & Gradient influence: $\frac{\partial m_v}{\partial A_e}$ \\
$\mathbf{A}_v, \mathbf{D}_v, \mathbf{\tilde{D}}_v$ & Local adjacency, degree matrix, perturbed degree matrix for node $v$ \\
$c_\mathbf{A}(v), c_{\mathbf{\tilde{A}}}(v)$ & Clustering coefficient (original/perturbed) for node $v$ \\
\bottomrule
\end{tabular}
\label{tab:notations}
\end{table}

\subsection{Minimality-Aware Post-Hoc Pruning: Algorithm 2}
\label{sec:alg2}

\rwOne{After the main optimization phase completes, the resulting perturbation $\Delta\mathbf{A}$ may contain extraneous edges introduced by approximation errors in gradient updates or overly cautious thresholding. To refine the explanation and ensure that every retained edge meaningfully contributes to the prediction flip, we apply a post-hoc minimality-aware pruning procedure, formalized in Algorithm~\ref{alg:prune}. This algorithm iteratively ranks edges in $\Delta\mathbf{A}$ by their estimated importance (measured via gradient magnitude) and removes them greedily in ascending order of this score. An edge is retained only if its removal would reverse the counterfactual prediction, ensuring that the final perturbation $\Delta\mathbf{A}^*$ is irreducible by construction.}

\rwOne{This refinement step improves the explanation's conciseness without requiring retraining or re-optimization. In practice, we observe that pruning reduces the average number of edits from $1.71$ to $1.62$, while maintaining the same level of predictive fidelity and plausibility. It also yields a significant runtime improvement (6.12s to 3.00s) due to the smaller edge set being evaluated downstream (Figure~\ref{effec_pruning}). Thus, pruning serves both an interpretability function by enforcing parsimony, and a practical one by improving efficiency.}

\begin{algorithm}[t]
\caption{Minimality Pruning}
\label{alg:prune}

\textbf{Input:} Perturbation $\Delta\mathbf{A}$, graph $G$, model $f$, target node $v$\\
\textbf{Output:} Minimal perturbation $\Delta\mathbf{A}^*$

\begin{enumerate}
  \item Initialize: $\Delta\mathbf{A}^* \gets \Delta\mathbf{A}$
  \item Rank edges in $\Delta\mathbf{A}^*$ by importance score $\psi_e$ (descending) \label{alg:line:rank}
  \item \textbf{for each} edge $e_i$ in ascending order of $\psi_e$: \label{alg:line:remove}
    \begin{enumerate}
      \item $\Delta\mathbf{A}' \gets \Delta\mathbf{A}^* \setminus \{e_i\}$ (tentatively remove)
      \item \textbf{if} $f(\mathbf{A} \odot \Delta\mathbf{A}',v) \neq f(\mathbf{A},v)$:
        \begin{enumerate}
          \item $\Delta\mathbf{A}^* \gets \Delta\mathbf{A}'$ (keep the smaller perturbation)
        \end{enumerate}
    \end{enumerate}
  \item \textbf{return} $\Delta\mathbf{A}^*$
\end{enumerate}

\end{algorithm}

\subsection{Evaluation Metrics}
\label{sec:metrics}
\begin{itemize}[noitemsep,nolistsep,leftmargin=*]
    \item \textbf{Misclassification Rate}: 
     It measures the fraction of predictions flipped by perturbations. Higher values indicate stronger disruption, consistent with the \emph{attack success rate} widely used in GNN adversarial attacks, such as Nettack \citep{zugner2018adversarial} and GOttack \citep{alom2025gottack}.
    \begin{align}
    \text{Misclassification Rate}=\tfrac{1}{N}\sum_{i=1}^{N}\mathbb{I}(\hat{y}_i^{1-m_i}\neq c_i),
    \end{align}
    where $N$ is the number of evaluated target nodes, $c_i=f(\mathbf{A}, \mathbf{X}, v_i)$ denotes the model-predicted class of target node $v_i$. $\hat{y}_i^{1-m_i}=f(\mathbf{\tilde A},\mathbf{X}, v_i)$, $\mathbf{\tilde A}$ is the perturbed adjacency, $m_i$ is the explanation mask (edges added/removed), $\mathbb{I}$ is the indicator function.
    
    \item \textbf{Fidelity}: 
    This metric measures the prediction confidence drop on the model’s predicted class $c_i$ \citep{bajaj2021robust}. Formally:
    \begin{align}
    \text{Fidelity}=\tfrac{1}{N}\sum_{i=1}^{N}\big(f(\mathbf{A}, \mathbf{X}, v_i)_{c_i}-f(\mathbf{\tilde A},\mathbf{X}, v_i)_{c_i}\big),
    \end{align}
    where $f(\mathbf{A}, \mathbf{X}, v)_{c}$ denotes the softmax probability assigned to class $c$. Unlike the binary Misclassification Rate, which captures label flips, Fidelity provides a finer-grained sensitivity analysis by quantifying how perturbations reduce the model’s confidence in its own prediction.

    \item \textbf{Explanation Size $\Delta\mathbf{E}$}: It represents the average number of structural modifications (including both edge additions and deletions) made per counterfactual explanation, calculated as: 
    \begin{align}
    \Delta\mathbf{E}=\tfrac{1}{n}\sum_{i=1}^{n} \Delta\mathbf{E_i}=\tfrac{1}{n}\sum_{i=1}^{n} (\Delta\mathbf{E_i}^+ + \Delta\mathbf{E_i}^-),
    \end{align}
    We report the average over successful counterfactuals $n$ since $\Delta\mathbf{E}_i$ is well-defined only when a valid counterfactual is generated, ensuring that the metric reflects the true complexity of feasible explanations rather than being diluted by failed cases \citep{lucic2022cf,DBLP:conf/www/TanGFGX0Z22}.
    Here, $\Delta\mathbf{E_i}$ represents the set of perturbed edges for node $v_i$. Smaller values indicate more compact and interpretable explanations. 

    \item \textbf{Plausibility}: This evaluates the human-interpretable quality of counterfactual explanations by assessing their realism and coherence with domain knowledge. The plausibility score is averaged across $n$ successful counterfactuals:
    \begin{align}
    \text{Plausibility}=\tfrac{1}{n}\sum_{i=1}^{n}S_{{plau}}^{(i)},\quad 
    S_{{plau}}^{(i)}=2\cdot\Big(1-\tfrac{1}{1+\exp(-k\cdot L_{{plau}}^{(i)})}\Big),
    \end{align}
    where $S_{{plau}}^{(i)}\in(0,1)$ is the plausibility score for target node $v_i$, $k$ is a scaling factor (default $k=1$), and $L_{{plau}}^{(i)}\in(0,\infty)$ encodes domain-specific constraints quantifying the realism of the counterfactual. Higher values indicate more plausible explanations.
    In our experiments, $L_{{plau}}^{(i)}$ is instantiated using the definition in Eq.\ref{eq:deganom} and Eq.\ref{eq:motifviol} in \S \ref{sec:plausibility}, ensuring consistency with our evaluation setup. More generally, $L_{{plau}}^{(i)}$ serves as a flexible placeholder that can incorporate task-specific structural and semantic constraints to assess the realism of counterfactuals in diverse domains.

    \item \textbf{Time Cost}: We record the average running time required in seconds to generate a counterfactual explanation for a single node, providing insights into the computational efficiency of different methods.
\end{itemize}

\begin{table}[t]
\caption{\rwOne{Performance of counterfactual explanations on \textbf{Cora} and GCN.}}
\label{c6-1}
\centering
\renewcommand{\arraystretch}{1.1}
\scriptsize
\begin{tabular}{lcccC{2.3cm}cc}
\hline
\textbf{Method} & 
\textbf{Base GNN} & 
\textbf{Misclass. $\uparrow$} & 
\textbf{Fidelity $\uparrow$} & 
\textbf{$\Delta\mathbf{E}(\mathbf{E}^+,\mathbf{E}^-)$ $\downarrow$} & 
\textbf{Plausibility $\uparrow$} & 
\textbf{Time (sec) $\downarrow$} \\
\hline
\hspace{1em}CF-GNNExplainer & GCN & 0.49$\pm$0.013 & 0.1060$\pm$0.0034  & 1.70$\pm$0.08 (0.00, 1.70) & 0.64$\pm$0.008 & 10.21$\pm$2.88 \\
\hspace{1em}\rwTwo{INDUCE} & GCN & 0.17$\pm$0.008 & 0.0256$\pm$0.0014  & 2.66$\pm$0.04 (1.50, 1.16) & 0.27$\pm$0.008 & 0.88$\pm$0.08 \\
\hspace{1em}\rwTwo{C2Explainer} & GCN & 0.61$\pm$0.029 & 0.2116$\pm$0.0046  & 1.95$\pm$0.03 (0.58, 1.37) & 0.68$\pm$0.008 & 8.25$\pm$0.12 \\
\hspace{1em}\rwTwo{CF$^2$} & GCN & 0.40$\pm$0.012 & 0.0813$\pm$0.0077  & 2.66$\pm$0.07 (0.00, 2.66) & 0.61$\pm$0.016 & 7.08$\pm$0.20 \\
\hspace{1em}\rwTwo{NSEG} & GCN & 0.33$\pm$0.012 & 0.0815$\pm$0.0077  & 3.26$\pm$0.15 (0.00, 3.26) & 0.58$\pm$0.008 & 3.47$\pm$0.14 \\
\hspace{1em}GNNExplainer    & GCN & 0.22$\pm$0.016 &0.0197$\pm$0.0150  & 2.58$\pm$0.13 (0.00, 2.58)  & 0.53$\pm$0.021  & 0.44$\pm$0.52 \\
\hspace{1em}PGExplainer     & GCN & 0.14$\pm$0.009 & -0.0010$\pm$0.0017  & 2.38$\pm$0.03 (0.00, 2.38)  & 0.53$\pm$0.005  & \textbf{0.04$\pm$0.02} \\
\hline 
\multicolumn{7}{l}{\textbf{Attack Models}} \\
\hspace{1em}Nettack         & GCN & 0.53$\pm$0.005 & 0.1484$\pm$0.0057  & 5.00$\pm$0.00 (3.86, 1.14)  & 0.13$\pm$0.005  & 3.36$\pm$0.85 \\
\hspace{1em}GOttack         & GCN & 0.53$\pm$0.005 & 0.1466$\pm$0.0043  & 5.00$\pm$0.00 (4.70, 0.30)  & 0.10$\pm$0.000  & 2.24$\pm$0.81 \\
\hline
\modelname(Ours)   & GCN & \textbf{0.72$\pm$0.008}     & \textbf{0.2336$\pm$0.0003} & \textbf{1.63$\pm$0.01 (0.90, 0.73)} & \textbf{0.75$\pm$0.008}      & 7.26$\pm$2.5 \\
\hline
\end{tabular}
\vspace{-1em}
\end{table}

\vspace{-0.5em}
\begin{table}[t]
\caption{\rwOne{Performance of counterfactual explanations on \textbf{BA-SHAPES} and GCN.}}
\label{c6-2}
\centering
\renewcommand{\arraystretch}{1.1}
\scriptsize
\begin{tabular}{lcccC{2.3cm}cc}
\hline
\textbf{Method} & 
\textbf{Base GNN} & 
\textbf{Misclass. $\uparrow$} & 
\textbf{Fidelity $\uparrow$} & 
\textbf{$\Delta\mathbf{E}(\mathbf{E}^+,\mathbf{E}^-)$ $\downarrow$} & 
\textbf{Plausibility $\uparrow$} & 
\textbf{Time (sec) $\downarrow$} \\
\hline
\multicolumn{7}{l}{\textbf{Explainers}} \\
\hspace{1em}CF-GNNExplainer & GCN & 0.64$\pm$0.017 & 0.3383$\pm$0.0079  & 1.33$\pm$0.20 (0.00, 1.33) & 0.57$\pm$0.012 & 11.30$\pm$3.72 \\
\hspace{1em}\rwTwo{INDUCE} & GCN & 0.70$\pm$0.005 & 0.3475$\pm$0.0140  & 1.64$\pm$0.02 (0.40, 1.24) & 0.28$\pm$0.008 & 0.25$\pm$0.02 \\
\hspace{1em}\rwTwo{C2Explainer} & GCN & 0.59$\pm$0.008 & 0.1426$\pm$0.0058  & 3.79$\pm$0.01 (0.02, 3.77) & 0.08$\pm$0.005 & 1.55$\pm$0.02 \\
\hspace{1em}\rwTwo{CF$^2$} & GCN & 0.58$\pm$0.012 & 0.1264$\pm$0.0052  & 2.87$\pm$0.05 (0.00, 2.87) & 0.07$\pm$0.008 & 5.22$\pm$0.17 \\
\hspace{1em}\rwTwo{NSEG} & GCN & 0.19$\pm$0.005 & 0.0269$\pm$0.0009  & 3.36$\pm$0.10 (0.00, 3.36) & 0.67$\pm$0.008 & 2.71$\pm$0.09 \\
\hspace{1em}GNNExplainer    & GCN & 0.65$\pm$0.022 &0.3055$\pm$0.0019  & 1.83$\pm$0.15 (0.00, 1.83)         & 0.34$\pm$0.009 & 0.81$\pm$1.03 \\
\hspace{1em}PGExplainer     & GCN & 0.73$\pm$0.031 & 0.3672$\pm$0.0015  & 1.45$\pm$0.05 (0.00, 1.45)         & 0.41$\pm$0.075         & \textbf{0.03$\pm$0.01} \\
\hline
\multicolumn{7}{l}{\textbf{Attack Models}} \\
\hspace{1em}Nettack         & GCN &0.64$\pm$0.005 &0.3526$\pm$0.0063  &5.00$\pm$0.00 (4.08, 0.92)              &0.22$\pm$0.0036              &0.89$\pm$0.38 \\
\hspace{1em}GOttack         & GCN &0.63$\pm$0.012 &0.3399$\pm$0.0081  &5.00$\pm$0.00 (4.30, 0.70)              &0.32$\pm$0.008              &0.73$\pm$0.22 \\
\hline
\modelname(Ours)   & GCN & \textbf{0.83$\pm$0.009}     & \textbf{0.4237$\pm$0.0118}  & \textbf{1.24$\pm$0.02 (1.21, 0.03)} & \textbf{0.71$\pm$0.000}      & 8.96$\pm$0.43 \\
\hline
\end{tabular}
\end{table}

\begin{table}[t]
\caption{\rwOne{Performance of counterfactual explanations on \textbf{TREE-CYCLES} and GCN.}}
\label{c6-3}
\centering
\renewcommand{\arraystretch}{1.1}
\scriptsize
\begin{tabular}{lcccC{2.3cm}cc}
\hline
\textbf{Method} & 
\textbf{Base GNN} & 
\textbf{Misclass. $\uparrow$} & 
\textbf{Fidelity $\uparrow$} & 
\textbf{$\Delta\mathbf{E}(\mathbf{E}^+,\mathbf{E}^-)$ $\downarrow$} & 
\textbf{Plausibility $\uparrow$} & 
\textbf{Time (sec) $\downarrow$} \\
\hline
\multicolumn{7}{l}{\textbf{Explainers}} \\
\hspace{1em}CF-GNNExplainer & GCN & 0.49$\pm$0.054 & 0.3437$\pm$0.0422  & 1.95$\pm$0.03 (0.00, 1.95) & 0.34$\pm$0.005 & 6.16$\pm$2.09 \\
\hspace{1em}\rwTwo{INDUCE} & GCN & 0.53$\pm$0.005 & 0.3194$\pm$0.0065  & 2.76$\pm$0.01 (0.89, 1.87) & 0.22$\pm$0.009 & \textbf{0.01$\pm$0.00} \\
\hspace{1em}\rwTwo{C2Explainer} & GCN & \textbf{0.74$\pm$0.005} & 0.2783$\pm$0.0116  & 3.09$\pm$0.07 (0.13, 2.96) & 0.14$\pm$0.005 & 10.15$\pm$0.32 \\
\hspace{1em}\rwTwo{CF$^2$} & GCN & 0.46$\pm$0.019 & 0.2531$\pm$0.0128  & 3.83$\pm$0.05 (0.00, 3.83) & 0.31$\pm$0.012 & 9.16$\pm$0.07 \\
\hspace{1em}\rwTwo{NSEG} & GCN & 0.45$\pm$0.012 & 0.2175$\pm$0.0094  & 3.76$\pm$0.07 (0.00, 3.76) & 0.29$\pm$0.012 & 5.34$\pm$0.20 \\
\hspace{1em}GNNExplainer    & GCN & 0.53$\pm$0.085 & 0.3608$\pm$0.0637  & 2.57$\pm$0.31 (0.00, 2.57) & 0.26$\pm$0.041 & 0.70$\pm$0.93 \\
\hspace{1em}PGExplainer     & GCN & 0.41$\pm$0.033 & 0.2733$\pm$0.0288  & 2.52$\pm$0.12 (0.00, 2.52) & 0.31$\pm$0.022 & \textbf{0.01$\pm$0.00} \\
\hline
\multicolumn{7}{l}{\textbf{Attack Models}} \\
\hspace{1em}Nettack         & GCN & 0.58$\pm$0.022 & \textbf{0.4508$\pm$0.0217}  & 5.00$\pm$0.00 (4.34, 0.66) & 0.27$\pm$0.099 & 0.58$\pm$0.17 \\
\hspace{1em}GOttack         & GCN & 0.18$\pm$0.005 & 0.1083$\pm$0.0033  & 5.00$\pm$0.00 (4.91, 0.09) & 0.21$\pm$0.016 & 0.41$\pm$0.09 \\
\hline
\modelname(Ours)    & GCN & 0.58$\pm$0.009 & 0.4052$\pm$0.0221  & \textbf{1.29$\pm$0.07 (0.69, 0.60)} & \textbf{0.64$\pm$0.009} & 2.98$\pm$1.41 \\
\hline
\end{tabular}
\end{table}

\begin{table}[t]
\caption{\rwOne{Performance of counterfactual explanations on \textbf{Loan-Decision} and GCN.}}
\label{c6-4}
\centering
\renewcommand{\arraystretch}{1.1}
\scriptsize
\begin{tabular}{lcccC{2.3cm}cc}
\hline
\textbf{Method} & 
\textbf{Base GNN} & 
\textbf{Misclass. $\uparrow$} & 
\textbf{Fidelity $\uparrow$} & 
\textbf{$\Delta\mathbf{E}(\mathbf{E}^+,\mathbf{E}^-)$ $\downarrow$} & 
\textbf{Plausibility $\uparrow$} & 
\textbf{Time (sec) $\downarrow$} \\
\hline
\multicolumn{7}{l}{\textbf{Explainers}} \\
\hspace{1em}CF-GNNExplainer & GCN & 0.45$\pm$0.092 & 0.2520$\pm$0.0490  & 1.35$\pm$0.20 (0.00, 1.35) & 0.53$\pm$0.038 & 56.00$\pm$7.04 \\
\hspace{1em}\rwTwo{INDUCE} & GCN & 0.18$\pm$0.102 & 0.0873$\pm$0.0516  & 3.23$\pm$0.76 (1.21, 2.01) & 0.39$\pm$0.169 & 1.68$\pm$0.21 \\
\hspace{1em}\rwTwo{C2Explainer} & GCN & 0.00$\pm$0.000 & --  & -- & -- & 39.88$\pm$1.23 \\
\hspace{1em}\rwTwo{CF$^2$} & GCN & 0.17$\pm$0.005 & 0.0774$\pm$0.0038  & 3.58$\pm$0.08 (0.00, 3.58) & 0.43$\pm$0.012 & 31.76$\pm$0.57 \\
\hspace{1em}\rwTwo{NSEG} & GCN & 0.18$\pm$0.008 & 0.0891$\pm$0.0022  & 3.74$\pm$0.11 (0.00, 3.74) & 0.45$\pm$0.012 & 26.43$\pm$0.50 \\
\hspace{1em}GNNExplainer    & GCN & 0.16$\pm$0.048 & 0.0438$\pm$0.0497  & 2.56$\pm$0.29 (0.00, 2.56) & 0.42$\pm$0.017 & 3.33$\pm$4.36 \\
\hspace{1em}PGExplainer     & GCN & 0.10$\pm$0.008 & 0.0281$\pm$0.0105  & 2.80$\pm$0.34 (0.00, 2.80) & 0.21$\pm$0.024 & \textbf{0.32$\pm$0.04} \\
\hline
\multicolumn{7}{l}{\textbf{Attack Models}} \\
\hspace{1em}Nettack         & GCN & 0.34$\pm$0.005 & 0.1685$\pm$0.0075  & 5.00$\pm$0.00 (3.01, 1.99)  & 0.24$\pm$0.017  & 1.16$\pm$0.43 \\
\hspace{1em}GOttack         & GCN & 0.35$\pm$0.005 & 0.1742$\pm$0.0053  & 5.00$\pm$0.00 (4.25, 0.75)  & 0.15$\pm$0.005  & 0.52$\pm$0.16 \\
\hline
\modelname(Ours)    & GCN & \textbf{0.68$\pm$0.024} & \textbf{0.3658$\pm$0.0171}  & \textbf{1.27$\pm$0.02 (0.38, 0.89)} & \textbf{0.67$\pm$0.026}  & 20.33$\pm$0.58 \\
\hline
\end{tabular}
\vspace{-1em}
\end{table}

\vspace{-0.5em}
\begin{table}[t]
\caption{\rwOne{Performance of counterfactual explanations on \textbf{ogbn-arxiv} and GCN.
}}
\label{c6-5}
\centering
\renewcommand{\arraystretch}{1.1}
\scriptsize
\begin{tabular}{lcccC{2.3cm}cc}
\hline
\textbf{Method} & 
\textbf{Base GNN} & 
\textbf{Misclass. $\uparrow$} & 
\textbf{Fidelity $\uparrow$} & 
\textbf{$\Delta\mathbf{E}(\mathbf{E}^+,\mathbf{E}^-)$ $\downarrow$} & 
\textbf{Plausibility $\uparrow$} & 
\textbf{Time (sec) $\downarrow$} \\
\hline
\multicolumn{7}{l}{\textbf{Explainers}} \\
\hspace{1em}CF-GNNExplainer & GCN & 0.45$\pm$0.033 & 0.0791$\pm$0.0072  & 1.56$\pm$0.06 (0.00, 1.56) & 0.66$\pm$0.005 & 7.16$\pm$1.71 \\
\hspace{1em}\rwTwo{INDUCE} & GCN & 0.28$\pm$0.066 & 0.0130$\pm$0.0015  & 2.17$\pm$0.24 (1.12, 1.05) & 0.41$\pm$0.009 & 0.15$\pm$0.02 \\
\hspace{1em}\rwTwo{C2Explainer} & GCN & 0.88$\pm$0.009 & 0.2552$\pm$0.0091  & 1.62$\pm$0.05 (0.64, 0.98) & 0.70$\pm$0.008 & 8.61$\pm$1.70 \\
\hspace{1em}\rwTwo{CF$^2$} & GCN & 0.42$\pm$0.012 & 0.0700$\pm$0.0100  & 2.85$\pm$0.04 (0.00, 2.85) & 0.48$\pm$0.021 & 6.99$\pm$0.05 \\
\hspace{1em}\rwTwo{NSEG} & GCN & 0.64$\pm$0.017 & 0.1956$\pm$0.0058  & 3.53$\pm$0.10 (0.00, 3.53) & 0.61$\pm$0.008 & 3.38$\pm$0.16 \\
\hspace{1em}GNNExplainer    & GCN & 0.33$\pm$0.014 & 0.0136$\pm$0.0040  & 2.22$\pm$0.07 (0.00, 2.22) & 0.63$\pm$0.008 & \textbf{0.10$\pm$0.01} \\
\hspace{1em}PGExplainer     & GCN & 0.26$\pm$0.012 & 0.0206$\pm$0.0056  & 2.25$\pm$0.13 (0.00, 2.25) & 0.59$\pm$0.022 & \textbf{0.10$\pm$0.05} \\
\hline
\multicolumn{7}{l}{\textbf{Attack Models}} \\
\hspace{1em}Nettack         & GCN & 0.86$\pm$0.009 & \textbf{0.3366$\pm$0.0059}  & 5.00$\pm$0.00 (4.38, 0.62) & 0.58$\pm$0.029 & 2.14$\pm$0.69 \\
\hspace{1em}GOttack         & GCN & 0.85$\pm$0.022 & 0.3251$\pm$0.0107  & 5.00$\pm$0.00 (5.00, 0.00) & 0.62$\pm$0.012 & 0.63$\pm$0.08 \\
\hline
\modelname(Ours)   & GCN & \textbf{0.90$\pm$0.012} & 0.3251$\pm$0.0023  & \textbf{1.20$\pm$0.05 (0.92, 0.28)} & \textbf{0.73$\pm$0.017} & 3.35$\pm$0.17 \\
\hline
\end{tabular}
\end{table}

\subsection{Experimental Results on Individual Datasets}
\label{sec:ind_data}

Across all datasets, \modelname flips the most target nodes while using very few edits. On Cora (Table \ref{c6-1}), \modelname achieves a misclassification rate of 0.72 with only 1.63 average edge changes, compared to only 0.53 for both Nettack and GOttack (each being forced to flip 5 edges). Our fidelity (0.2336) and plausibility (0.75) are also the highest. In contrast, PGExplainer is extremely fast (0.04s) but flips almost no nodes, and attack methods (Nettack/GOttack) flip all 5 edges but yield very low plausibility ($\approx 0.1–0.13$). A similar pattern holds on BA-Shapes (Table \ref{c6-2}) and Tree-Cycles (Table \ref{c6-3}), which are motif-based synthetic graphs. On BA-Shapes, \modelname attains 0.83 misclassification with $\Delta E$=1.24 and plausibility 0.71, clearly outperforming others; on Tree-Cycles, it achieves 0.58 misclassification vs. 0.74 for C2Explainer, but with far higher plausibility (0.64 vs. 0.14) and much smaller edits ($\Delta E$=1.29 vs. 3.09). These synthetic benchmarks have no node features and explicit motif structures, and \modelname reliably discovers the minimal motif changes needed.

On the Loan-Decision social graph (Table \ref{c6-4}), \modelname again dominates: 0.68 misclassification (vs. $\le 0.45$ for others) and highest fidelity (0.3658) with only $\Delta E$=1.27. Finally, on the large real ogbn-arxiv network (Table \ref{c6-5}), \modelname flips 0.90 fraction of nodes vs. 0.85–0.88 for C2Explainer and attacks, yet uses just $\Delta E$=1.20 edges (C2Explainer uses 1.62 and attacks use 5) and achieves plausibility 0.73 (vs. 0.58–0.70). The ogbn-arxiv dataset is a citation graph of CS papers with 128-dimensional features and 40 classes, confirming \modelname scales to large, feature-rich graphs. 

\rwTwo{Finally, to further assess generalization beyond homophilic benchmarks, we include experiments on the heterophilic Chameleon dataset (Table~\ref{c6-6}). Chameleon poses additional challenges due to feature heterophily and non-community structure. Nonetheless, \modelname achieves the highest misclassification rate (0.84) and fidelity (0.2595) while requiring only 1.64 edits—substantially fewer than other baselines, which require more edits (vs. 1.81-5). \modelname also maintains high plausibility (0.76), comparable to the best-performing explainers yet with significantly stronger effectiveness.}

\rwTwo{These results confirm that \modelname remains robust across both homophilic and heterophilic graphs.} In summary, our method consistently finds compact counterfactual edits that flip more predictions than baselines, yielding higher fidelity while preserving realistic graph structure.

\subsection{Experimental results with Graph Transformer and GAT}
\label{sec:other_gnn}

Tables~\ref{ap-gt-1} and~\ref{ap-gat-1} further demonstrate that \modelname remains consistently superior on both Graph Transformer \citep{shi2020masked}  and GAT \citep{velickovic2018graph} backbones. On Graph Transformer (Table~\ref{ap-gt-1}), our method achieves the highest misclassification rate (0.44) and plausibility (0.50), while also maintaining competitive edit compactness ($\Delta \mathbf{E}=1.66$). On GAT (Table~\ref{ap-gat-1}), \modelname shows an even clearer margin, boosting misclassification to 0.47 and plausibility to 0.65, outperforming all baselines by a large gap. These results confirm that our mask optimization generalizes beyond GCNs, remaining stable and effective across different architectures, including both attention-based and transformer-based GNNs.

Moreover, Tables \ref{ap-gt-1} and \ref{ap-gat-1} show that applying CF-GNNExplainer to attention-based models such as GAT and Graph Transformer often results in unstable mask optimization. This instability arises because, unlike GCN where the normalized adjacency enters linearly into the convolution allowing effective gradient flow from the loss to the mask, attention-based architectures compute edge attention coefficients via nonlinear transformations (LeakyReLU, softmax). Any mask applied to edge weights is absorbed and scaled by $\alpha_{ij}(1-\alpha_{ij}) \ll 1$, leading to vanishing gradient signals and preventing the identification of meaningful counterfactual edges.

In {\sf CF-GNNExplainer}, the adjacency mask $P$ is treated as a continuous parameter (after a sigmoid), which scales the edge weight or serves as an edge attribute. In attention-based models, these edge attributes enter the computation of attention logits $z_{ij}$:
\begin{equation}
z_{ij} = s_{ij} + b \cdot e_{ij}, \quad e_{ij} = \sigma(P_{ij}),
\end{equation}
where $s_{ij}$ is a feature-derived score, $b$ is a scalar, and $\sigma$ is the sigmoid. The normalized attention coefficient is
\begin{equation}
\alpha_{ij} = \frac{\exp(z_{ij})}{\sum_{k \in \mathcal{N}(i)} \exp(z_{ik})}.
\end{equation}
The gradient of the loss $L$ with respect to $P_{ij}$ is then
\begin{equation}
\frac{\partial L}{\partial P_{ij}} = 
\frac{\partial L}{\partial \alpha_{ij}} \cdot
\underbrace{\frac{\partial \alpha_{ij}}{\partial z_{ij}}}_{\alpha_{ij}(1-\alpha_{ij})} \cdot
\underbrace{\frac{\partial z_{ij}}{\partial e_{ij}}}_{b} \cdot
\underbrace{\frac{\partial e_{ij}}{\partial P_{ij}}}_{\sigma'(P_{ij})}.
\end{equation}
The critical term is the Jacobian of the softmax:
\[
\frac{\partial \alpha_{ij}}{\partial z_{ij}} = \alpha_{ij}(1-\alpha_{ij}).
\]
When the degree of node $i$ is $N$ and neighbors are similar, $\alpha_{ij} \approx 1/N$, thus $\alpha_{ij}(1-\alpha_{ij}) \approx \mathcal{O}(1/N)$.
This means that the mask gradient is strongly diluted by $1/N$, which is further multiplied by the sigmoid derivative $\sigma'(P_{ij})$. As a result, the gradient magnitude quickly vanishes, especially for high-degree nodes. This explains why {\sf CF-GNNExplainer} struggles to optimize adjacency masks in attention-based models.

\begin{table}[t]
\caption{\rwOne{Performance of counterfactual explanations on \textbf{Loan-Decision} and Graph Transformer.}}
\label{ap-gt-1}
\centering
\renewcommand{\arraystretch}{1.1}
\scriptsize
\begin{tabular}{lcccC{2.3cm}cc}
\hline
\textbf{Method} & 
\textbf{Base GNN} & 
\textbf{Misclass. $\uparrow$} & 
\textbf{Fidelity $\uparrow$} & 
\textbf{$\Delta\mathbf{E}(\mathbf{E}^+,\mathbf{E}^-)$ $\downarrow$} & 
\textbf{Plausibility $\uparrow$} & 
\textbf{Time (sec) $\downarrow$} \\
\hline
CF-GNNExplainer & Graph Trans. & -- & --  & -- & -- & -- \\
\rwTwo{INDUCE} & Graph Trans. & 0.02$\pm$0.012 & 0.1264$\pm$0.0044  & 3.57$\pm$0.43 (2.10, 1.47) & 0.29$\pm$0.000 & 0.35$\pm$0.09 \\
\rwTwo{C2Explainer} & Graph Trans. & 0.06$\pm$0.009 & 0.1549$\pm$0.0165  & 2.95$\pm$0.16 (2.36, 0.59) & \textbf{0.53$\pm$0.033} & 8.02$\pm$0.29 \\
\rwTwo{CF$^2$} & Graph Trans. & 0.05$\pm$0.005 & 0.1402$\pm$0.0052  & 3.61$\pm$0.07 (0.00, 3.61) & 0.46$\pm$0.021 & 7.03$\pm$0.07 \\
\rwTwo{NSEG} & Graph Trans. & 0.07$\pm$0.012 & 0.1701$\pm$0.0245  & 3.89$\pm$0.07 (0.00, 3.89) & 0.44$\pm$0.022 & 4.10$\pm$0.10 \\
GNNExplainer    & Graph Trans. & 0.30$\pm$0.039 & 0.2452$\pm$0.0513  & 2.99$\pm$0.27 (0.00, 2.00) & 0.36$\pm$0.012 & 4.77$\pm$3.45 \\
PGExplainer     & Graph Trans. & 0.39$\pm$0.007 & 0.3187$\pm$0.0141  & \textbf{1.66$\pm$0.27 (0.00, 1.66)} & 0.45$\pm$0.031 & \textbf{0.05$\pm$0.03} \\
Nettack         & Graph Trans. & 0.32$\pm$0.004 & 0.2509$\pm$0.0101  & 5.00$\pm$0.00 (4.03, 0.97) & 0.22$\pm$0.016 & 1.03$\pm$0.48 \\
GOttack         & Graph Trans. & 0.31$\pm$0.006 & 0.2420$\pm$0.0072  & 5.00$\pm$0.00 (4.81, 0.19) & 0.30$\pm$0.005 & 0.66$\pm$0.21 \\
\hline
\modelname(Ours)   & Graph Trans. & \textbf{0.44$\pm$0.035} & \textbf{0.3563$\pm$0.0120}  & \textbf{1.66$\pm$0.01 (0.85, 0.81)} & 0.50$\pm$0.024 & 8.07$\pm$2.41 \\
\hline
\end{tabular}
\end{table}

\begin{table}[t]
\caption{\rwOne{Performance of counterfactual explanations on \textbf{Loan-Decision} and GAT.}}
\label{ap-gat-1}
\centering
\renewcommand{\arraystretch}{1.1}
\scriptsize
\begin{tabular}{lcccC{2.3cm}cc}
\hline
\textbf{Method} & 
\textbf{Base GNN} & 
\textbf{Misclass. $\uparrow$} & 
\textbf{Fidelity $\uparrow$} & 
\textbf{$\Delta\mathbf{E}(\mathbf{E}^+,\mathbf{E}^-)$ $\downarrow$} & 
\textbf{Plausibility $\uparrow$} & 
\textbf{Time (sec) $\downarrow$} \\
\hline
CF-GNNExplainer & GAT & -- & --  & -- & -- & -- \\
\rwTwo{INDUCE} & GAT & 0.18$\pm$0.115 & 0.0420$\pm$0.0280  & 2.64$\pm$0.18 (1.63, 1.01) & 0.38$\pm$0.061 & 0.24$\pm$0.01 \\
\rwTwo{C2Explainer} & GAT & 0.08$\pm$0.021 & 0.0108$\pm$0.0055  & 3.03$\pm$0.18 (2.46, 0.57) & 0.52$\pm$0.033 & 6.42$\pm$0.09 \\
\rwTwo{CF$^2$} & GAT & 0.07$\pm$0.008 & 0.0126$\pm$0.0029  & 3.35$\pm$0.04 (0.00, 3.35) & 0.47$\pm$0.017 & 5.11$\pm$0.09 \\
\rwTwo{NSEG} & GAT & 0.12$\pm$0.008 & 0.0297$\pm$0.0014  & 3.68$\pm$0.09 (0.00, 3.68) & 0.45$\pm$0.016 & 3.13$\pm$0.07 \\
GNNExplainer    & GAT & 0.01$\pm$0.005 & 0.0002$\pm$0.0315  & 3.00$\pm$0.00 (0.00, 3.00) & 0.34$\pm$0.015 & 3.40$\pm$1.67 \\
PGExplainer     & GAT & 0.08$\pm$0.032 & 0.0057$\pm$0.0012  & 3.39$\pm$0.00 (0.00, 3.39) & 0.46$\pm$0.092 & \textbf{0.06$\pm$0.01} \\
Nettack         & GAT & 0.41$\pm$0.006 & 0.0781$\pm$0.0078  & 5.00$\pm$0.12 (3.93, 1.07) & 0.20$\pm$0.014 & 1.01$\pm$0.51 \\
GOttack         & GAT & 0.32$\pm$0.007 & 0.0689$\pm$0.0043  & 5.00$\pm$0.04 (4.80, 0.20) & 0.18$\pm$0.005 & 0.60$\pm$0.18 \\
\hline
\modelname(Ours)   & GAT & \textbf{0.47$\pm$0.021} & \textbf{0.0892$\pm$0.0193}  & \textbf{1.58$\pm$0.03 (0.67, 0.91)} & \textbf{0.65$\pm$0.019} & 4.32$\pm$1.95 \\
\hline
\end{tabular}
\vspace{-1em}
\end{table}

\vspace{-0.5em}
\begin{table}[t]
\caption{\rwOne{Performance of counterfactual explanations on \textbf{Chameleon} and GCN.
}}
\label{c6-6}
\centering
\renewcommand{\arraystretch}{1.1}
\scriptsize
\begin{tabular}{lcccC{2.3cm}cc}
\hline
\textbf{Method} & 
\textbf{Base GNN} & 
\textbf{Misclass. $\uparrow$} & 
\textbf{Fidelity $\uparrow$} & 
\textbf{$\Delta\mathbf{E}(\mathbf{E}^+,\mathbf{E}^-)$ $\downarrow$} & 
\textbf{Plausibility $\uparrow$} & 
\textbf{Time (sec) $\downarrow$} \\
\hline
\multicolumn{7}{l}{\textbf{Explainers}} \\
\hspace{1em}CF-GNNExplainer & GCN & 0.39$\pm$0.033 & 0.0858$\pm$0.0134  & 1.81$\pm$0.02 (0.00, 1.81) & \textbf{0.81$\pm$0.054} & 375.39$\pm$8.38 \\
\hspace{1em}\rwTwo{INDUCE} & GCN & 0.03$\pm$0.005 & 0.0096$\pm$0.0004  & 2.33$\pm$0.10 (1.50, 0.83) & 0.44$\pm$0.033 & 18.23$\pm$0.29 \\
\hspace{1em}\rwTwo{C2Explainer} & GCN & 0.09$\pm$0.016 & 0.0596$\pm$0.0026  & 2.61$\pm$0.07 (2.17, 0.44) & 0.76$\pm$0.008 & 85.69$\pm$1.53 \\
\hspace{1em}\rwTwo{CF$^2$} & GCN & 0.21$\pm$0.025 & 0.0879$\pm$0.0091  & 3.51$\pm$0.12 (0.00, 3.51) & 0.62$\pm$0.024 & 73.31$\pm$2.40 \\
\hspace{1em}\rwTwo{NSEG} & GCN & 0.30$\pm$0.025 & 0.0870$\pm$0.0069  & 3.86$\pm$0.09 (0.00, 3.86) & 0.61$\pm$0.021 & 42.25$\pm$1.54 \\
\hspace{1em}GNNExplainer    & GCN & 0.01$\pm$0.005 & 0.0046$\pm$0.0023  & 3.50$\pm$0.41 (0.00, 3.50) & 0.53$\pm$0.008 & 11.28$\pm$0.28 \\
\hspace{1em}PGExplainer     & GCN & 0.01$\pm$0.005 & 0.0003$\pm$0.0002  & 2.83$\pm$2.09 (0.00, 2.83) & 0.45$\pm$0.328 & \textbf{1.79$\pm$0.10} \\
\hline
\multicolumn{7}{l}{\textbf{Attack Models}} \\
\hspace{1em}Nettack         & GCN & 0.51$\pm$0.008 & 0.1806$\pm$0.0008  & 5.00$\pm$0.00 (4.73, 0.27) & 0.29$\pm$0.012 & 38.23$\pm$0.70 \\
\hspace{1em}GOttack         & GCN & 0.52$\pm$0.029 & 0.1944$\pm$0.0010  & 5.00$\pm$0.00 (4.87, 0.13) & 0.43$\pm$0.012 & 9.26$\pm$0.13 \\
\hline
\modelname(Ours)            & GCN & \textbf{0.84$\pm$0.008} & \textbf{0.2595$\pm$0.0004}  & \textbf{1.64$\pm$0.01 (1.16, 0.48)} & 0.76$\pm$0.005 & 134.73$\pm$4.15 \\
\hline
\end{tabular}
\end{table}

\begin{table}[ht]
\caption{Ablation Study on \textbf{Cora} and GCN.}
\label{ablation-study}
\centering
\renewcommand{\arraystretch}{1.1}
\scriptsize
\begin{tabular}{lcccC{2.3cm}cc}
\hline
\textbf{Method} & 
\textbf{Base GNN} & 
\textbf{Misclass. $\uparrow$} & 
\textbf{Fidelity $\uparrow$} & 
\textbf{$\Delta\mathbf{E}(\mathbf{E}^+,\mathbf{E}^-)$ $\downarrow$} & 
\textbf{Plausibility $\uparrow$} & 
\textbf{Time (sec) $\downarrow$} \\
\hline
w/o $\mathcal{L}_{dist}$ & GCN & 0.70 & 0.2360 & 1.66 (0.76, 0.90) & 0.71 & 7.42 \\
w/o $\mathcal{L}_{plau}$ & GCN & 0.68 & 0.2225 & \textbf{1.57 (0.72, 0.85)} & 0.71 & 6.82 \\
w/o $\mathcal{L}_{dist}$ and $\mathcal{L}_{plau}$ & GCN & 0.69 & \textbf{0.2469} & 1.60 (0.59, 1.01) & 0.68 & 6.22 \\
\hline
\modelname(Ours) & GCN & \textbf{0.71} & 0.2336 & 1.62 (0.92, 0.70) & \textbf{0.75} & \textbf{3.80} \\
\hline
\end{tabular}
\vspace{-1em}
\end{table}

\begin{figure}[t]
  \centering
  \includegraphics[width=0.85 \textwidth]{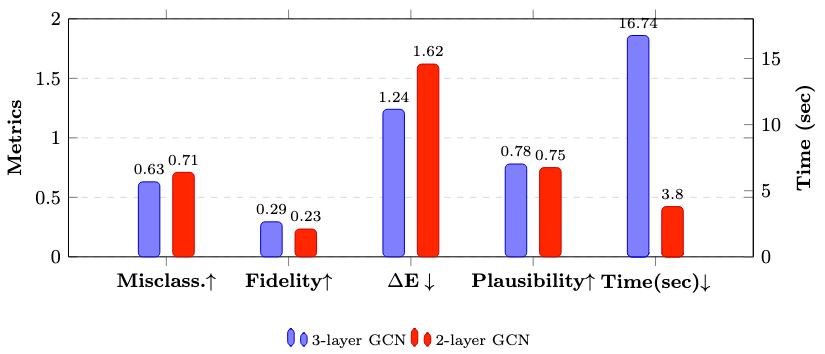}
  \caption{Performance of counterfactual explanations vs. the number of GNN layers: The results demonstrate sensitivity w.r.t. the number of hops for the local structure surrounding the target node.}
  \vspace{-5px}
  \label{sensitive to layer}
\end{figure}

\begin{figure}[t]
  \centering
  \includegraphics[width=0.7 \textwidth]{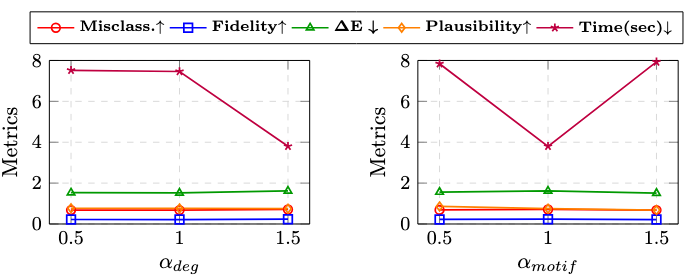}
  \caption{Sensitivity w.r.t. Hyperparameters $\alpha_{deg}$ and $\alpha_{motif}$.}
  \vspace{-5px}
  \label{hyper_sensitive}
\end{figure}

\begin{figure}[ht]
  \centering
  \includegraphics[width=0.85 \textwidth]{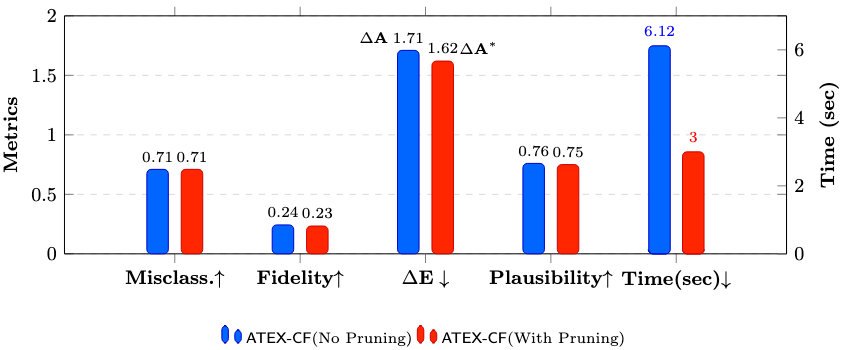}
  \caption{Effectiveness of Post-Hoc Pruning.}
  \vspace{-5px}
  \label{effec_pruning}
\end{figure}

In contrast, \modelname uses a \emph{signed, discrete mask} $M_{ij} \in \{-1,0,+1\}$ combined with a straight-through estimator (STE) for backpropagation:
\begin{equation}
\tilde{A}_{ij} = \tilde{A}_{ij}^{\text{discrete}} + \big(M_{ij}^{\text{cont}} - \mathrm{sg}(M_{ij}^{\text{cont}})\big), \quad
\frac{\partial L}{\partial M_{ij}^{\text{cont}}} \approx \frac{\partial L}{\partial \tilde{A}_{ij}},
\end{equation}
where $\tilde{A}_{ij}$ is the perturbed adjacency entry for edge $(i,j)$, $\tilde{A}_{ij}^{\text{discrete}}$ is the discrete (binary) adjacency entry, $M_{ij}^{\text{cont}}$ is the continuous mask, $\mathrm{sg}(\cdot)$ denotes the stop-gradient operator that blocks gradients, and $L$ is the model loss.
This allows gradients to capture the \emph{finite-difference effect} of adding or deleting an edge on the loss, without attenuation from softmax or nonlinearities. As a result, the signed mask optimization remains stable and effective across both GCNs and attention-based models, enabling reliable counterfactual explanations.

\subsection{A Case study on Counterfactuals}

\rwTwo{To illustrate the refinement effect of pruning, Figure~\ref{fig:before_after_explanation} compares the initial and final counterfactual explanations for a target node. Before the pruning, the explanation includes several edges of varying influence, some of which are unnecessary for achieving the prediction flip. The attack procedure discards these redundant edits by greedily testing their necessity. In this case, the final counterfactual achieves its goal by deleting the edge $(1978, 1306)$ and adding a single edge $(1978, 190)$. This minimal perturbation is sufficient to alter the model's decision.}
\begin{figure}[ht]
\centering
\begin{minipage}{0.48\textwidth}
  \centering
  \includegraphics[width=\linewidth]{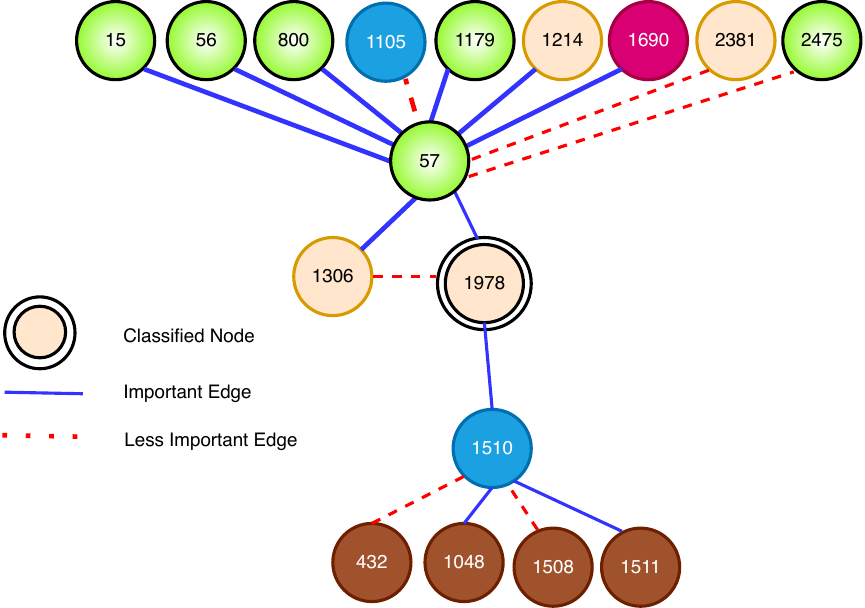}
  \caption*{(a) Before attack}
\end{minipage}
\hfill
\begin{minipage}{0.48\textwidth}
  \centering
  \includegraphics[width=\linewidth]{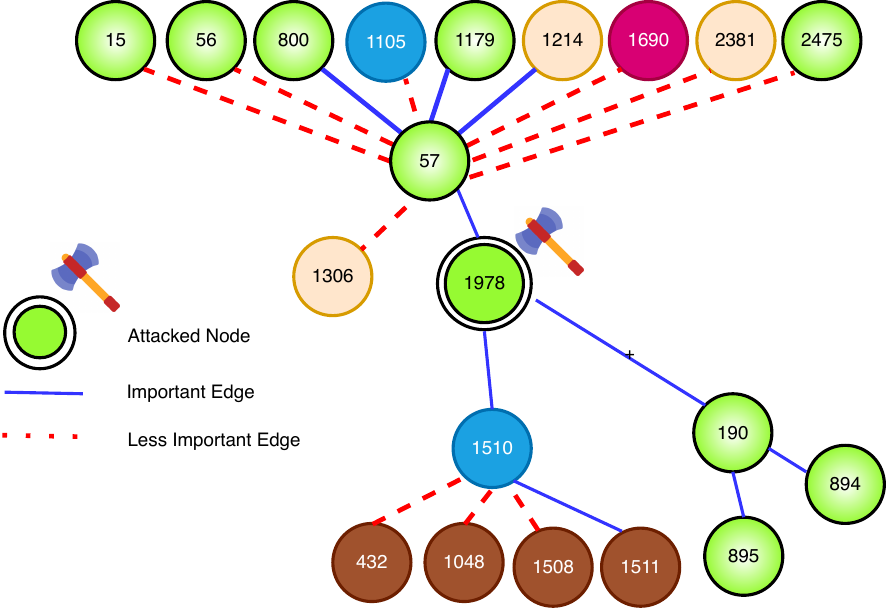}
  \caption*{(b) After attack}
\end{minipage}
\caption{\rwTwo{Visual comparison of the changed graph when the edge 1978 is classified (panel a) and its counterfactual edges (panel b). Node colors indicate class membership. The initial classification includes both important and less important edges. The counterfactual deletes edge (1978, 1306) and adds edge (1978, 190) to successfully flip the prediction for 1978.}}
\label{fig:before_after_explanation}
\end{figure}

\subsection{Ablation Study} 
\label{sec:ablation}

Table \ref{ablation-study} shows the effect of removing each loss on Cora. Removing $\mathcal{L}_{dist}$ reduces misclassification slightly (0.71 → 0.70), leading to larger edit sets (1.62 → 1.66) and lower plausibility (0.75 → 0.71), indicating that edit minimality is compromised. Omitting the plausibility loss ($\mathcal{L}_{plau}$) yields the smallest edit size (1.57), but severely hurts misclassification, dropping the rate to 0.68, as edits no longer respect semantic structure. Removing both losses reduces misclassification and plausibility. These findings demonstrate that $\mathcal{L}_{dist}$ enforces concise edits, $\mathcal{L}_{plau}$ preserves semantic plausibility, and their combination in \modelname achieves the best overall balance across all metrics.

\paragraph{Ablation on Attack Candidate Selection.} 
\rwOne{We conducted ablations by replacing the GOttack-based candidate selection with two alternative strategies: (i) random sampling and (ii) Nettack-based edge proposals. These modifications isolate the impact of attack-informed candidate generation on counterfactual quality.}

\begin{table}[ht]
\centering
\caption{\rwOne{Effect of attack candidate generation strategy on counterfactual explanations (Cora, GCN, $\kappa = 5$).}}
\label{tab:candidate_ablation}
\resizebox{\linewidth}{!}{
\begin{tabular}{lcccccc}
\toprule
Method & Base GNN & Misclass. & Fidelity & $\Delta E$ (E$^+$, E$^-$) & Plausibility & Time (sec) \\
\midrule
ATEX-CF + Random Sampling & GCN & 0.46 & 0.17 & 1.71 (0.74, 0.97) & 0.67 & 7.44 \\
ATEX-CF + Nettack & GCN & 0.67 & 0.22 & 1.62 (0.75, 0.87) & 0.72 & 8.97 \\
ATEX-CF + GOttack (ours) & GCN & 0.69 & 0.22 & 1.66 (0.76, 0.90) & 0.72 & 4.71 \\
\bottomrule
\end{tabular}}
\end{table}

\rwOne{As summarized in Table~\ref{tab:candidate_ablation}, using GOttack candidates yields the highest misclassification rate and fidelity under the same budget, outperforming both random and Nettack-based baselines. Compared to random sampling, GOttack offers a $\approx 50$\% increase in flip rate with higher plausibility and lower runtime. This demonstrates that attack-informed candidates meaningfully improve counterfactual effectiveness.}

\subsection{Sensitivity Analysis} 
\label{sec:sensitivity}

We analyze the key hyperparameters using Cora. \textbf{Search depth ($l$):} Varying the number of hops for local structure surrounding the target node shows diminishing returns beyond local context. Figure \ref{sensitive to layer} depicts that going from $l=2$ to $l=3$ yields only marginal improvements in fidelity, edits, and plausibility (e.g., +0.06 fidelity, -0.38 $\Delta \mathbf{E}$, +0.03 plausibility), while increasing computation time and dropping misclassification. This indicates that depth-2 captures sufficient structure for effective counterfactuals.
\textbf{Hyperparameters ($\alpha_{deg}, \alpha_{motif}$):} We vary the weights of degree-anomaly and motif-anomaly terms in plausibility loss. Figure \ref{hyper_sensitive} demonstrates that \modelname is robust across a range of values (e.g. $\alpha=0.5$–$1.5$); misclassification and fidelity remain high. Very low $\alpha$ removes the corresponding regularizer and slightly degrades plausibility, while very high $\alpha$ yields negligible gains but more aggressive edits. In practice, moderate $\alpha$ maximizes fidelity and plausibility together.

\subsection{Impact and Benefit of the Pruning Strategy}
\label{sec:pruning_exp}
We also evaluate the impact of our candidate-edge pruning on GCN with the Cora dataset.
As shown in Figure~\ref{effec_pruning}, pruning yields more concise explanations by reducing redundant edits ($\Delta \mathbf{A}=1.71 \to 1.62$), while maintaining nearly identical predictive accuracy (misclassification $=0.71$) and plausibility (0.76 vs.~0.75). Runtime is significantly reduced (6.12s vs.~3.00s), confirming that pruning improves explanatory minimality and efficiency without sacrificing fidelity or plausibility.

\rwThree{We found optimizing the loss directly preferable to pruning for two reasons. First, the loss computation is efficient and integrates naturally into gradient-based optimization. Second, pruning after proposing a perturbation can be counterproductive in graph settings, where a single edge change can change the receptive fields of many nodes.}  

\rwThree{We illustrate the second point with an example. Suppose that the model proposes the sequence of edits: $+e_7$, $+e_8$, $-e_2$. If $+e_7$ is deemed implausible and later pruned, the remaining edits $+e_8$ and $-e_2$ may no longer have any effect, as their impact depended on $+e_7$ being applied. This leads to wasted computation and potentially invalid counterfactuals. In contrast, our loss-based optimization penalizes implausible edits during training. As a result, $+e_7$ would be discarded early, and the model would avoid proposing follow-up edits that rely on its presence. This allows us to generate counterfactuals more efficiently and with greater consistency.}

\rwThree{In the first reason, we noted that plausibility loss is computationally cheap. This is because ATEX-CF restricts all computations to the (($l+1$)-hop subgraph) of the target node (e.g., GCN’s 3 layers + 1), the maximum degree ($d_{max}$) is naturally bounded by this local neighborhood. }

\rwThree{Under this setting, the plausibility loss remains computationally lightweight. Even in worst-case large graphs such as ogbn-arxiv, the 4-hop neighborhood typically has ($d_{max} \approx 50, n_{sub} \approx 1000$), amounting to $(1000 \times 50^2 = 2.5\times 10^6)$ operations, still negligible compared to a single GNN forward pass $O(L\times |E|d+L\times n_{sub}\times d^2)$.  Other datasets’ numbers are given in Table~\ref{tab:dataset_density}. }

\begin{table}[ht]
\centering
\caption{Graph statistics and localized density parameters across datasets.}
\label{tab:dataset_density}
\resizebox{\linewidth}{!}{
\begin{tabular}{lccccc}
\toprule
Dataset & Global Avg. Degree & Global Max Degree & Typical 4-hop Local Max & Reasonable $d_{\max}$ & Typical $n_{\text{sub}}$ \\
\midrule
BA-SHAPES     & $\sim$4     & $<$10   & 4--6   & 5    & 50--100 \\
Cora          & $\sim$3.7   & $\sim$40 & 12--16 & 15   & 100--300 \\
Ogbn-arxiv    & $\sim$13    & $\sim$1800 & 40--60 & 50   & 500--1000 \\
\bottomrule
\end{tabular}}
\end{table}

\rwThree{Pruning is costly because it requires a GNN forward pass to detect which edges to perturb. Integrating the plausibility loss helps us avoid unnecessary, costly passes for candidates that would be pruned at the end.
}

\subsection{Proof and Evidence for the Hypotheses}
\label{sec:hypotheses}

Throughout this section, let $v$ be the target node under analysis. We use $s_G(v)$ to denote the logit of the target class for $v$, $f_G(v)$ for the predicted label of $v$, and $m_v$ for the margin between the logit of $v$’s true class and the highest competing class.

By definition, $CFEx(G)$ is an inclusion-minimal set of edge modifications (additions or deletions) such that applying them flips $f$’s prediction for node $v$. Minimal means that no proper subset of those modifications is sufficient to flip the prediction. Let us denote this modification set by $F := CFEx(G)$. 
$$
f_{G \oplus F}(v) \neq f_G(v),
$$
but for any strict subset $F' \subsetneq F$, 
$$
f_{G \oplus F'}(v) = f_G(v).
$$
We assume that the influence function of $f$ over edge sets is submodular, so the marginal effect of adding or removing an edge diminishes as more modifications are applied (influence functions on graphs are often modeled as submodular~\citep{krause2007near, borgs2014maximizing}). This submodularity assumption implies that the minimal counterfactual explanation set $F$ is unique, which ensures that alignment between the attack-selected edges and the explanation subgraph is well defined, i.e., the top-$k$ edges chosen by the attack coincide with the uniquely defined set $F$ rather than an arbitrary minimal set. 

When multiple such minimal sets exist, we fix a canonical choice by breaking ties, for example, by selecting the lexicographically smallest edge set. Intuitively, $F$ captures the single most crucial evidence subgraph in $G$ supporting the original prediction. 

For any edge $e$ and set of edges $S$, define the conditioned marginal effect as
$\Delta_e f(G \cup S;\,v) := f_{\,G \cup S \cup \{e\}}(v) - f_{\,G \cup S}(v).$

\refstepcounter{HypothesisCounter}%
\subsubsection{Hypothesis H\arabic{HypothesisCounter}: Edge Gradient Attack Alignment}\label{sub:proof-H1}

\setcounter{hypothesis}{0}
\begin{hypothesis}[Restated]\label{hyp:H1-restated}
Let $G = (A, X)$ be an input graph and $f$ a pre-trained GNN classifier. 
For a target node $v$, let $\Delta G(E^+)$ denote the set of added edges in an evasion attack that flips the prediction of $f$, and let $CFEx(G)$ denote the counterfactual explanation graph of the graph $G$. 
Then, the graph similarity between $\Delta G(E^+)$ and $CFEx(G)$:

$$\text{Sim}(\Delta G(E^+), CFEx(G)) \approx c,\quad 0<<c<1,$$
where $\text{Sim}(\cdot,\cdot)$ denotes a graph similarity measure by graph edit distance, maximum common subgraph, and graph embedding vectors, and $c$ is a positive score, indicating non-trivial overlap between the attack edges and the explanation graph.
\end{hypothesis}

\noindent\textbf{Proof Sketch:} Edges with the largest gradient influence on the target node’s logit margin are the most potent for adversarial attacks. Formally, for a target node $v$ with margin $m_v(A)$ and edge gradients $g_e=\partial m_v/\partial A_e$, suppose $e_1$ and $e_2$ are two candidate edges (with $e_1$ either currently present or absent depending on the attack type, and similarly for $e_2$). If $|g_{e_1}| > |g_{e_2}|$, then flipping $e_1$ (adding it if $g_{e_1}<0$ or removing it if $g_{e_1}>0$) yields a larger drop in $m_v$ than flipping $e_2$. In particular, a Projected Gradient Descent (PGD) attack will primarily select edges from among those with the highest $|g_e|$, aligning adversarial modifications with the gradient-based explanation subgraph. Intuitively, the gradient $g_e$ indicates how sensitively the margin $m_v$ changes with respect to edge $e$. A large-magnitude gradient $|g_e|$ means that a small change in $A_e$ has a big effect on $m_v$. In a 2-layer GCN with ReLU, the model is piecewise linear, so locally $m_v$ changes approximately linearly with $A_e$. Thus, the edge with the largest $|g_e|$ produces the steepest change in $m_v$ when perturbed. A PGD adversarial attack, which follows the gradient of the loss (or negative margin), will therefore choose the edge with the most negative gradient (for additions) or the most positive gradient (for deletions) to maximally decrease the margin. In essence, explanation methods pick out these high-$|g_e|$ edges as important, and the attacker targets the very same edges to flip the prediction.

\begin{proof} Consider the target node $v$ with true class $y_v$ and margin $m_v(A) = z_{y_v}(A,v) - \max_{c\ne y_v} z_c(A,v)$. Let $g_e = \frac{\partial m_v}{\partial A_e}$ be the gradient influence of edge $e$ on the margin. We analyze edge addition, and show that larger $|g_e|$ implies a greater reduction in margin when $e$ is perturbed:

%\noindent\textbf{Case 1:} Edge removal (E- attack). Suppose $e=(i,j)$ is an existing edge in the graph ($A_e=1$) that appears in the factual explanation subgraph for node $v$. If $g_e > 0$, this edge supports the current prediction (increasing $m_v$). Removing $e$ (setting $A_e$ from 1 to 0) will decrease the margin. By definition of the derivative, $g_e = \partial m_v/\partial A_e$ quantifies the initial rate of change: for a small decrease $\Delta A_e$, $m_v$ decreases by about $g_e \cdot \Delta A_e$. Because our GCN is piecewise linear (ReLU activation), if removing $e$ does not trigger a different ReLU activation pattern, the change in margin is exactly linear: $m_v(A_{-e}) - m_v(A) = -g_e$ (here $A_{-e}$ denotes the adjacency with $e$ removed). Even if the removal causes a new linear region, we assume local Lipschitz continuity (addressed formally in H5) so that the margin drop remains proportional to $g_e$. Now if $|g_{e_1}| > |g_{e_2}|$ for two edges $e_1,e_2$ with positive gradients, $e_1$ provides a larger supportive influence. Removing $e_1$ thus yields a larger margin decrease than removing $e_2$. In other words, $\Delta m_v$ from removing $e_1$ is more negative than that from removing $e_2$. An optimal adversary will therefore target the edge with the highest positive gradient first, aligning with the explanation’s top factual edge.

\noindent\textbf{Edge addition (E+ attack).} Suppose $e=(i,j)$ is a non-existent edge ($A_e=0$). If $g_e < 0$, then $e$ is a detrimental or counterfactual edge for the current prediction: increasing $A_e$ (adding this edge) will lower the margin $m_v$. In a small continuous relaxation of $A_e$, $m_v$ would decrease by about $|g_e|\cdot \Delta A_e$. For the actual discrete addition ($A_e:0 \to 1$), the change $m_v(A_{+e}) - m_v(A)$ will be approximately $g_e$ (since $g_e$ is negative, this is a drop in margin). Because our GCN is piecewise linear (ReLU activation), adding $e$ causes a margin change on the order of $g_e$. If $|g_{e_1}| > |g_{e_2}|$ for two absent edges with negative gradients, adding $e_1$ produces a larger margin drop than adding $e_2$. Thus, an adversary performing PGD will add the edge with the most negative gradient first, which is precisely the top edge identified by a counterfactual explanation method.

The attacker’s choice of edge corresponds to the edge with the largest $|g_e|$ that reduces the margin (negative $g_e$ for addition). By repeating this argument iteratively (considering the next most influential edge after the first, and so on), one can see that an attack adding/removing $k$ edges will choose the $k$ edges with highest gradient magnitudes that contribute to lowering $m_v$. Therefore, the set of edges targeted by the PGD attack aligns with the gradient-based counterfactual explanation subgraph (which consists of edges with the largest $|g_e|$). This establishes that ranking edges by $|g_e|$ is equivalent to ranking them by adversarial effectiveness, proving the hypothesis. 
\label{proof:H1-margin}
\end{proof}

\begin{table}[t]
\caption{\textbf{Attacks and Counterfactuals.} The structural similarity between evasion attack edges $\Delta \mathbf{G}$ (mainly additions $\Delta \mathbf{E}^+$ from GOttack) and instance-level factual explanations $Ex(G')$ from GNNExplainer on post-attack graph $G'$. 280 target nodes are correctly classified in the original graph $G$. Budget = 5. GCN (2-layer), \textbf{Cora} dataset.}
\label{tab:sim_correct}
\begin{center}
\renewcommand{\arraystretch}{1.2}
\begin{tabular}{lccc}
\toprule
\textbf{Metric} & \textbf{All (280)} & \textbf{Attack Success (225)} & \textbf{Attack Fail (55)} \\
\midrule
GED$\downarrow$  & 0.38 & 0.37 & 0.41 \\
MCS$\uparrow$  & 0.31 & 0.33 & 0.24 \\
GEV$\uparrow$  & 0.72 & 0.80 & 0.39 \\
\bottomrule
\end{tabular}
\end{center}
\end{table}

\noindent\textbf{Empirical Evidence for Hypothesis~\ref{hyp:H1-restated}}

The results in Table~\ref{tab:sim_correct} support Hypothesis~\ref{hyp:H1-restated}, which posits a high structural overlap between the attacker’s perturbation $\Delta G$ and the counterfactual explanation $CFEx(G)$ produced by pre-attack explanation methods. Notice that here we consider the instance-level factual explanations $Ex(G')$ from GNNExplainer \citep{ying2019gnnexplainer} on the post-attack graph $G'$ as a proxy for the counterfactual explanation $CFEx(G)$ produced by pre-attack explanation methods. This is because the state-of-the-art counterfactual explainers generally do not support edge addition.  

In both correctly and incorrectly predicted instances, the Graph Edit Distance (GED) remains moderate ($\approx 0.38$), and the Maximum Common Subgraph (MCS) similarity is non-negligible, particularly for successful attacks. Notably, Graph Embedding Vector (GEV) similarity reaches $0.88$ for misclassified nodes and $0.80$ for successful attacks on correctly predicted nodes, indicating substantial alignment in the embedded subgraph structure. In other words, Table~\ref{tab:sim_correct} shows that similarity between attack perturbations $\Delta G$ and counterfactual explanations $CFEx(G)$ depends strongly on attack outcome. For successful attacks, distances such as GED are lower (lower is better) and similarities such as MCS and GEV are higher (higher is better), while for failed attacks, the opposite holds. In other words, when the attack succeeds, the perturbations align closely with counterfactual explanations, whereas in failed cases the overlap weakens. This pattern offers evidence that effective adversarial edits not only cause misclassification but also resemble the explanatory structures that counterfactual  methods would identify.

\subsubsection{Propositions on Counterfactual Completeness via Attack-Informed Additions}

In principle, for the completeness of our hypothesis, one would like to prove that edge additions \textquote{always} yield a successful counterfactual attack, which would strengthen our claim that unifying attacks and counterfactuals is universally beneficial, even when counterfactuals alone fail. Unfortunately, this cannot be guaranteed, since the data may lack any node whose connection to the target would flip its label. Instead, we establish a next-best guarantee: When sufficiently informative opposite-class nodes exist, additions can flip the label while deletions cannot. State-of-the-art counterfactual explanations may overlook such opportunities, but attack algorithms are designed to exploit them.

Let $f$ be a GNN classifier and let $v\in V$ be a target node with $f_G(v)=y$. Throughout, $s_G(v)$ denotes a real valued class $y$ score for $v$, $f_G(v)$ denotes the predicted label, and $m_v$ denotes the margin for class $y$ at $v$. $w_{vu}$ is the weight assigned by the model to the contribution of neighbor $u$ when aggregating into the score of node $v$. Fix a one versus rest view for class $y$ and use the decision rule $f_G(v)=y$ if and only if $s_G(v)>0$. Assume an additive, degree-independent neighborhood model

$$
s_G(v)=bias_v+\sum_{u\in\mathcal{N}(v)} w_{vu}\,r_u,
$$

with $w_{vu}\ge 0$, where $r_u$ is the contribution aligned with class $y$. This additive influence model abstracts away normalization and attention redistribution, but shows the monotonic nature of homophilic neighborhoods. While GNNs are more complex, we observe empirically that their behavior is consistent with the model’s prediction: deletion of a few homophilic neighbors rarely flips predictions, whereas a small number of targeted additions frequently does (as evidenced in the addition attacks of Gottack~\cite{alom2025gottack} and Nettack~\cite{zugner2018adversarial}). 

We assume homophily in the immediate neighborhood so that $f_G(u)=y$ for all $u\in\mathcal{N}(v)$, hence $r_u\ge 0$ for all incident neighbors. No term in $s_G(v)$ is rescaled by $|\mathcal{N}(v)|$.

In this setting, deletion and addition have asymmetric effects. Deleting any number of incident edges can only remove nonnegative summands, while adding edges to informative opposite class nodes can introduce negative summands. The next two propositions formalize this.

\begin{proposition}[Deletion Infeasibility]
Let $G'=G\setminus S$ for some strict subset $S\subsetneq{(v,u):u\in\mathcal{N}(v)}$. If

$$
bias_v+\min_{u\in\mathcal{N}(v)} w_{vu}\,r_u\;>\;0,
$$

where $bias_v$ is a bias term for node $v$’s own features, $r_u$ is the contribution from neighbor $u$’s features, aligned with class $y$, and $w_{vu} \ge 0$ is the scalar weight that measures how strongly neighbor $u$ influences $v$’s score. Then $f_{G'}(v)=y$. In words, as long as at least one incident neighbor remains, the score stays positive, and the label does not change.
\end{proposition}

\emph{Argument.} The smallest possible post-deletion score over all strict subsets occurs when only the least contributing neighbor of $v$ remains. This score equals $b_v+\min_{u} w_{vu}r_u$, which is positive by assumption, hence $f_{G'}(v)=y$.

\begin{proposition}[Addition Sufficiency]
Suppose there exists a set of candidate nodes $C$ with $f_G(u)\ne y$ such that for each $u\in C$, adding the edge $(v,u)$ decreases the score by at least a fixed amount $\gamma>0$:

$$
s_{G\cup\{(v,u)\}}(v)\;\le\;s_G(v)-\gamma.
$$

Let $m_v=s_G(v)>0$. Then there exists a set $E^+\subseteq{(v,u):u\in C}$ with

$$
|E^+|\;\le\;\lceil m_v/\gamma\rceil
$$

such that $f_{G\cup E^+}(v)\ne y$. Thus, a small number of informative additions flips the prediction.
\end{proposition}

\emph{Argument.} Each addition reduces the score by at least $\gamma$. After $k=\lceil m_v/\gamma\rceil$ additions, the score is nonpositive, which changes the predicted label.

\begin{corollary}[Budgeted reachability and strict advantage of additions]
Let $\mathcal{R}_{\mathrm{del}}(k)=\{G\setminus S:\,S\subseteq\{(v,u):u\in\mathcal{N}(v)\},\,|S|\le k\}$ 
and $\mathcal{R}_{\mathrm{add}}(k)=\{G\cup E^+:\,E^+\subseteq\{(v,u):u\in C\},\,|E^+|\le k\}$. 
Under the assumptions above, if
$$
bias_v+\min_{u\in\mathcal{N}(v)} w_{vu}r_u>0,
$$
then for every $k<|\mathcal{N}(v)|$ there is no graph in $\mathcal{R}_{\mathrm{del}}(k)$ that flips $v$’s label. 
If, in addition, there exists $\gamma>0$ such that each $(v,u)$ with $u\in C$ decreases $s_G(v)$ by at least $\gamma$, 
then with $k_+=\lceil m_v/\gamma\rceil$ there exists $G^+\in\mathcal{R}_{\mathrm{add}}(k_+)$ that flips $v$’s label. 
Consequently, whenever $k_+<|\mathcal{N}(v)|$, the set of counterfactuals reachable by at most $k_+$ additions is nonempty 
while the set reachable by at most $k_+$ deletions is empty, hence additions strictly dominate deletions under equal edit budgets.
\end{corollary}

\emph{Argument.} The deletion claim follows from the deletion infeasibility proposition. 
The addition claim follows from the addition sufficiency proposition with $k_+=\lceil m_v/\gamma\rceil$. 
If $k_+<|\mathcal{N}(v)|$, then $\mathcal{R}_{\mathrm{add}}(k_+)$ contains a prediction flipping graph while $\mathcal{R}_{\mathrm{del}}(k_+)$ does not.

\begin{corollary}[Edit cost and latent stability]
Let $d_{\mathrm{edit}}$ be the edge edit distance. 
Any witnessing addition set $E^+$ has $d_{\mathrm{edit}}(G,G\cup E^+)=|E^+|\le\lceil m_v/\gamma\rceil$. 
If a node-level embedding map $\psi(v;G)$ is $L$-Lipschitz with respect to incident edge edits at $v$, then
$$
\|\psi(v;G)-\psi(v;G\cup E^+)\|_2\le L\,|E^+|\le L\,\lceil m_v/\gamma\rceil.
$$
Thus the latent perturbation can be bounded linearly by the required number of additions.
\end{corollary}

\emph{Remark.} The strict advantage condition $k_+<|\mathcal{N}(v)|$ is testable from estimates of $m_v$ and per edge gains. 
If $k_+\ge|\mathcal{N}(v)|$, the theory is agnostic about dominance, but the separation holds whenever the margin-to-gain ratio 
is small relative to the neighborhood size.

\noindent\textbf{Empirical Evidence for the Counterfactual Completeness}

\begin{table}[t]
\caption{Failure rate of deletion-based counterfactual explanations for correctly predicted target nodes (\textbf{Cora}, 2-layer GCN).}
\label{tab:cf_failure_correct}
\begin{center}
\renewcommand{\arraystretch}{1.2}
\begin{tabular}{lccc}
\toprule
\textbf{Method} & \textbf{Total Nodes} & \textbf{Has CF Explanation} & \textbf{No CF Explanation} \\
\midrule
CF-GNNExplainer & 280 & 54 & 226 \\
\bottomrule
\end{tabular}
\end{center}
\end{table}

\begin{table}[t]
\caption{Failure rate of deletion-based counterfactual explanations for incorrectly predicted target nodes (\textbf{Cora}, 2-layer GCN).}
\label{tab:cf_failure_incorrect}
\begin{center}
\renewcommand{\arraystretch}{1.2}
\begin{tabular}{lccc}
\toprule
\textbf{Method} & \textbf{Total Nodes} & \textbf{Has CF Explanation} & \textbf{No CF Explanation} \\
\midrule
CF-GNNExplainer & 220 & 76 & 144 \\
\bottomrule
\end{tabular}
\end{center}
\end{table}

Tables~\ref{tab:cf_failure_correct} and~\ref{tab:cf_failure_incorrect} show how often CF‑GNNExplainer \citep{lucic2022cf} fails to generate deletion‑only counterfactual explanations under two conditions.

In Table~\ref{tab:cf_failure_correct} (correctly predicted target nodes), out of 280 test nodes, only the “\textsc{Has CF Explanation}” column reports 54 nodes ($\approx 19\%$) for which a deletion‐based counterfactual exists; the remaining 226 nodes ($\approx 81\%$) are in the “\textsc{No CF Explanation}” column.  
Similarly, in Table~\ref{tab:cf_failure_incorrect} (misclassified nodes), 76 out of 220 nodes ($\approx 35\%$) have a deletion‐only counterfactual, while 144 nodes ($\approx 65\%$) do not.

These high failure rates support our theoretical propositions and corollaries: namely, that there are many nodes for which deletion‐based counterfactuals are infeasible. These empirical gaps justify the necessity of incorporating attack‑informed edge additions to recover explanations for those nodes.  

\subsection{More Results on Asymmetric Costs of Edge Perturbations}
\label{sec:asym_appendix}

\rwOne{For completeness, we report additional results on asymmetric perturbation costs under smaller budgets ($\kappa=5$ and $\kappa=10$). As shown in Tables~\ref{tab:asym5} and \ref{tab:asym10}, the qualitative behavior observed in Section~\ref{sec:asym_main} remains consistent across budgets. Increasing the relative cost of edge additions systematically shifts the optimizer toward deletion-heavy solutions, reducing the number of added edges and eliminating them entirely when $C$ exceeds perturbation budget $\kappa$. This induces predictable effects on counterfactual quality: both misclassification success and fidelity decline as additions are discouraged, whereas plausibility remains comparatively stable until additions are nearly disallowed. These results confirm that cost asymmetry provides a reliable control mechanism for shaping the perturbation profile, while also highlighting that excessively large $C$ can over-constrain the search space and impair explanation effectiveness.}

\begin{table}[ht]
\centering
\caption{\rwOne{Counterfactual performance under asymmetric addition cost $C$ with $\kappa=5$.}}
\label{tab:asym5}
\resizebox{\textwidth}{!}{
\begin{tabular}{ccccccc}
\toprule
Addition Cost & Deletion Cost & Misclass. & Fidelity & $\Delta E$ (E$^+$, E$^-$) & Plausibility & Time (sec) \\
\midrule
0.5 & 1.0 & 0.69  & 0.23    & 1.69(0.78,0.91)  & 0.72  & 10.5 \\
0.8 & 1.0 & 0.69  & 0.22    & 1.66(0.77,0.89)  & 0.72  & 10.4 \\
1.0 & 1.0 & 0.69  & 0.22    & 1.66(0.76,0.90)  & 0.72  & 10.2 \\
2.0 & 1.0 & 0.67  & 0.22    & 1.59(0.65,0.94)  & 0.71  & 10.5 \\
2.5 & 1.0 & 0.64  & 0.20    & 1.49(0.62,0.87)  & 0.71  & 10.7 \\
5.0 & 1.0 & 0.53  & 0.15    & 1.35(0.49,0.86)  & 0.68  & 11.2 \\
6.0 & 1.0 & 0.43  & 0.10    & 1.59(0.00,1.59)  & 0.61  & 11.3 \\
\bottomrule
\end{tabular}}
\end{table}

\begin{table}[H]
\centering
\caption{\rwOne{Counterfactual performance under asymmetric addition cost $C$ with $\kappa=10$.}}
\label{tab:asym10}
\resizebox{\textwidth}{!}{
\begin{tabular}{ccccccc}
\toprule
Addition Cost & Deletion Cost & Misclass. & Fidelity & $\Delta E$ (E$^+$, E$^-$) & Plausibility & Time (sec) \\
\midrule
0.5 & 1.0 & 0.70  & 0.23    & 1.78(0.78,1.00)  & 0.72  & 5.9 \\
0.8 & 1.0 & 0.70  & 0.23    & 1.78(0.77,1.01)  & 0.71  & 6.0 \\
1.0 & 1.0 & 0.70  & 0.23    & 1.78(0.77,1.02)  & 0.71  & 7.4 \\
3.0 & 1.0 & 0.70  & 0.23    & 1.77(0.66,1.11)  & 0.69  & 7.7 \\
5.0 & 1.0 & 0.68  & 0.22    & 1.71(0.61,1.10)  & 0.69  & 7.8 \\
7.0 & 1.0 & 0.68  & 0.22    & 1.70(0.59,1.10)  & 0.69  & 7.8 \\
10  & 1.0 & 0.54  & 0.15    & 1.42(0.49,0.93)  & 0.68  & 8.2 \\
11  & 1.0 & 0.42  & 0.10    & 1.78(0.00,1.78)  & 0.62  & 8.5 \\
\bottomrule
\end{tabular}}
\end{table}

\subsection{Acknowledgment}

YZ and AK acknowledge support from the Novo Nordisk Foundation grant (NNF 22OC0072415). SBY acknowledges support from the National Science and Foundation of China (62402082) and Scientific and Technological Research Program of Chongqing Municipal Education Commission (KJQN202400637).

\end{document}